\documentclass[lettersize,journal]{IEEEtran}
\usepackage{amsmath,amsfonts}
\usepackage{algorithmic}
\usepackage{algorithm}
\usepackage{array}
\usepackage[caption=false,font=normalsize,labelfont=sf,textfont=sf]{subfig}
\usepackage{textcomp}
\usepackage{stfloats}
\usepackage{url}
\usepackage{verbatim}
\usepackage{graphicx}
\usepackage{cite}
\usepackage{amsthm}

\theoremstyle{plain}
\newtheorem{theorem}{Theorem}
\newtheorem{lemma}{Lemma}

\usepackage{caption} 
\usepackage[dvipsnames]{xcolor}
\usepackage{multirow}
\usepackage{amsmath}
\usepackage{amssymb}
\usepackage{xcolor}
\usepackage{array} 
\usepackage{makecell}
\usepackage{float}
\definecolor{iccvblue}{rgb}{0.21,0.49,0.74}
\usepackage{subcaption}

\definecolor{mypink}{RGB}{255, 10, 255}
\definecolor{myorange}{RGB}{255,105,14}
\usepackage{xspace}
\usepackage{cleveref}
\usepackage{comment}
\usepackage{xcolor}
\usepackage{booktabs, nicematrix}

\usepackage{enumitem}

\newcommand{\dataset}{PKR-QA\xspace}
\newcommand{\pkg}{PKG\xspace}
\newcommand{\kg}{PKG\xspace}
\newcommand{\kgs}{PKGS\xspace}

\usepackage{tikz-cd}
\usepackage{stackengine}

\newcommand\xnorightarrow[1]{%
  \stackengine{0.5pt}{$\xrightarrow{#1}$}{\textcolor{red}{\Huge$\times$}}{O}{c}{F}{T}{L}%
}

\usepackage[most]{tcolorbox}
\usepackage[table]{xcolor}
\usepackage{colortbl}

\newcommand{\un}[1]{\underline{#1}}

\newcommand{\svlmzero}[0]{Zero-shot}
\newcommand{\svlmzeroplus}[0]{VLM+P.VRL\xspace}

\newcommand{\skshot}[0]{QA~training\xspace}
\newcommand{\skgtrain}[0]{KG-training\xspace}
\newcommand{\skgkshot}[0]{KG+QA~training\xspace}

\newcommand{\skml}[0]{KML\xspace}

\newcommand{\kml}{KML\xspace}
\newcommand{\INGP}{IGP\xspace}
\newcommand{\dlong}[0]{Procedural Knowledge Reasoning Question Answering}
\newcommand{\dshort}[0]{PKR-QA}
\newcommand{\rotate}{RotatE\xspace}
\newcommand{\transe}{TransE\xspace}
\newcommand{\transh}{TransH\xspace}
\newcommand{\kmlfclip}{KML-F-CLIP\xspace}
\newcommand{\kmlclip}{KML-CLIP\xspace}
\newcommand{\kmlrand}{KML-Rand\xspace}

\usepackage{tikz}
\def\checkmark{\tikz\fill[scale=0.4](0,.35) -- (.25,0) -- (1,.7) -- (.25,.15) -- cycle;}

\begin{document}

\title{Neuro Symbolic Knowledge Reasoning for \\Procedural Video Question Answering}

\author{
\IEEEauthorblockN{
Basura~Fernando\textsuperscript{1,2,3},
Thanh-Son~Nguyen\textsuperscript{2},
Hong~Yang\textsuperscript{1,2},
Tzeh~Yuan~Neoh\textsuperscript{1,2},
Hao~Zhang\textsuperscript{1,2},
Ee~Yeo~Keat\textsuperscript{1,2}
}
\IEEEauthorblockA{\textsuperscript{1}Centre for Frontier AI Research, A*STAR, Singapore.}\\
\IEEEauthorblockA{\textsuperscript{2}Institute of High-Performance Computing, A*STAR, Singapore.}\\
\IEEEauthorblockA{\textsuperscript{3}College of Computing and Data Science, NTU, Singapore.}\\
\IEEEauthorblockA{
Fernando\_Basura@a-star.edu.sg, Nguyen\_Thanh\_Son@a-star.edu.sg
}
}



\maketitle

\begin{abstract}
In this work we present Knowledge Module Learning (KML) to understand and reason over procedural tasks that requires models to learn structured and compositional procedural knowledge. 
KML is a neurosymbolic framework that learns relation categories within a knowledge graph as neural knowledge modules and composes them into executable reasoning programs generated by large language models (LLMs). 
Each module encodes a specific procedural relation capturing how each entity type such as tools are related to steps,  purpose of each tool, and steps of each task.
Given a question conditioned on a task shown in a video, then KML
performs multistep reasoning with transparent, traceable intermediate states.
Our theoretical analysis demonstrated two desired properties of KML.
KML satisfy strong optimal conditions for modelling KG relations as neural mappings, providing strong foundations for generalizable procedural reasoning.
It also shows a bound on the expected error when it performs multistep reasoning.
To evaluate this model, we construct a large \textit{procedural knowledge graph} (PKG) consisting of diverse instructional domains by integrating the COIN instructional video dataset, and COIN ontology, commonsense relations from ConceptNet, and structured extractions from LLMs, followed by expert verification. We then generate question and answer pairs by applying graph-traversal templates over the PKG, constructing the \dataset benchmark for procedural knowledge reasoning.
Experiments show that KML improves structured reasoning performance while providing interpretable step-by-step traces, outperforming LLM-only and black-box neural baselines. Code is publicly available at \url{https://github.com/LUNAProject22/KML}.
\end{abstract}



\section{Introduction}
\label{sec:intro}

The procedural knowledge is the ability to understand and reason through how, why, and what sequence of steps are executed to accomplish a task. Procedural knowledge is fundamentally structured, sequential, and goal-oriented. It demands the handling of temporal dependencies and causal relationships in each step, while accounting for tool affordances and the functional intent behind every step.
In humans, this knowledge is developed via lifelong experience and hierarchical abstraction. This process facilitates high-level generalization, allowing individuals to adapt to environmental constraints (such as missing instrumentation), predict future states (such as mental time travel~\cite{suddendorf1997mental}), and provide explainable rationales for their actions~\cite{thoe2022developing}.
Replicating these sophisticated reasoning capabilities is critical for the advancement of autonomous systems. For embodied AI to achieve reliability in highly complex environments including household robotics, industrial maintenance, and clinical healthcare, it is well understood that they must transit from pattern recognition to a deep understanding of the procedural logic that can supports complex human tasks~\cite{ashutosh2024video,zhou2023procedure,nagasinghe2024not}.


Recent advances in Vision–Language Models (VLMs) have demonstrated impressive performance on a wide range of multimodal reasoning tasks, including video understanding~\cite{li2024mvbench} and question answering~\cite{lei2018tvqa}. However, despite their strong pattern recognition and linguistic reasoning abilities, current VLMs largely rely on implicit, unstructured reasoning encoded within large neural networks. As a result, their reasoning processes are difficult to interpret, hard to constrain to domain-specific logic, and prone to hallucination when faced with incomplete or out-of-distribution procedural knowledge. This limitation becomes particularly critical in procedural domains, where correctness, reliability, and traceability of reasoning are as important as final accuracy.

Several recent efforts attempt to improve the transparency of LLM/VLM reasoning through prompting strategies such as chain-of-thought, tree-of-thought, and graph-of-thought reasoning~\cite{wei2022chain,besta2024graph,yao2023tree}. While these approaches encourage models to externalize intermediate reasoning steps, the underlying reasoning variables remain unconstrained and are not grounded in explicit symbolic structures or explicit logical reasoning. Consequently, reasoning and execution are still tightly coupled within the model, limiting the ability to verify, debug, or systematically improve reasoning behaviour, especially when domain-specific procedural constraints must be respected. Consequently, hallucinations and under explained results are still potential to occur. 

To address this, methods like ViperGPT~\cite{suris2023vipergpt} decouple reasoning from execution by generating executable programs. 
In parallel, neuro-symbolic (NS) approaches aim to combine the flexibility of neural representations with the structure and interpretability of symbolic reasoning. However, many existing NS systems either rely on manually defined symbolic operators, which do not scale well to large procedural domains, or on fixed knowledge graph embeddings that struggle with compositional multi-hop reasoning and uncertainty propagation. Moreover, most prior NS approaches have focused on static scenes or short-horizon reasoning, leaving procedural, multi-step, temporally grounded reasoning in videos under explored~\cite{johnson2017inferring,mascharka2018transparency,andreas2016learning,perez2018film,hudson2019learning,hudson2018compositional,chen2021meta,endo2023motion,li2025imore}.

To address these challenges, we propose Knowledge Module Learning (KML), a neuro-symbolic framework that explicitly models procedural relations as learnable neural modules and composes them into executable reasoning programs generated by large language models. Unlike black-box VLM reasoning, KML separates what reasoning steps should be performed (program synthesis) from how each step is executed (relation-specific neural knowledge modules). Each module corresponds to a well-defined relation in a procedural knowledge graph (PKG) and is trained to map between entity types (e.g., steps to tools, tools to purposes) using contrastive learning.

This design yields several key advantages. First, reasoning is explicit and traceable, producing interpretable intermediate states that correspond to meaningful procedural entities. Second, reasoning is constrained by the schema of the PKG, ensuring alignment with domain knowledge while still allowing flexibility through neural generalization. Third, KML naturally supports uncertainty-aware multi-hop reasoning, enabling it to handle ambiguous grounding and incomplete observations—common challenges in real-world video understanding.

To rigorously evaluate procedural knowledge reasoning, we introduce PKR-QA, a large-scale benchmark constructed from a curated procedural knowledge graph integrating instructional video annotations, common-sense knowledge, and LLM-assisted extraction followed by human verification. Unlike existing video QA benchmarks that emphasize visual recognition or situational common-sense~\cite{wu2024star,yu2019activitynet,li2024mvbench,xu2017video,wang2024sok}, PKR-QA focuses on task-centric procedural reasoning, requiring models to infer purposes, anticipate steps, reason over alternatives, and generalize across tasks.

Finally, beyond empirical performance, we provide a theoretical analysis of KML, establishing sufficient conditions under which neural knowledge modules learn relation mappings that separate valid from invalid entities and deriving bounds on reasoning error under multi-step composition. This theoretical grounding, combined with extensive ablations and qualitative analyses, positions KML as a principled and scalable approach to trustworthy procedural reasoning.

This is an extended version of our AAAI 2026 paper~\cite{nguyen2026pkr}.
We have included additional theoretical analysis of KML which allows us to better understand the capabilities of KML.
Further, we have included additional ablation studies showing the impact of each component in our VQA architecture.
We also extended the KML to perform logic-based operators such as logical AND OR and NOT and evaluate the performance.

\section{Related Work}
\label{sec:related}
\begin{figure*}[t]
    \centering
    \includegraphics[width=.9\textwidth]{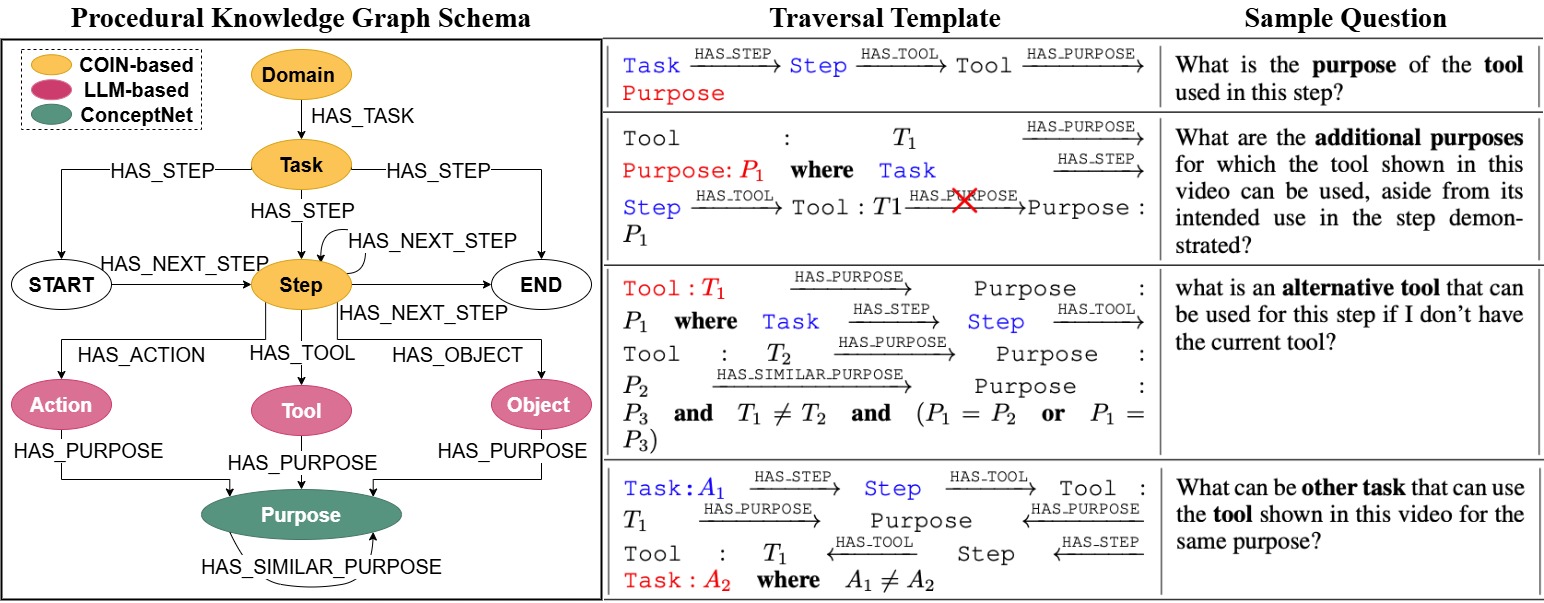}
    \caption{(Left) Schema of the Procedural Knowledge Graph (\pkg) showing the high-level abstraction of \pkg. In the middle are examples of Traversal Templates that define reasoning patterns over \pkg to generate question-answer pairs. Corresponding example questions are shown on the right. In the traversal templates, \textcolor{blue}{blue text} indicates information grounded in the input video, while \textcolor{red}{red text} denotes the target answer node. 
    }
    \label{fig:dataset_creation}
\end{figure*}

\subsection{Procedural and Instructional Video Understanding}

Understanding instructional and procedural videos has been an active area of research in computer vision, with early works focusing on representation learning from narrated videos and step recognition~\cite{miech2020end,ashutosh2024video,zhou2023procedure,zhong2023learning, nagasinghe2024not,chang2020procedure,sun2022plate}.
More recently tasks such as goal-inference~\cite{ee2025deduce,wan2025infer,roy2024predicting} for procedural tasks in both robotics and vision are getting popular.
Datasets such as COIN~\cite{tang2019coin} and ActivityNet-QA~\cite{yu2019activitynet} have enabled progress in recognizing steps, actions, and temporal structure in instructional content. Subsequent works introduced procedure-aware representations and planning-oriented models that explicitly model step transitions, task graphs, and long-horizon dependencies.
While these approaches significantly advance procedural perception, they primarily focus on recognizing or predicting steps rather than reasoning over explicit procedural knowledge such as tool affordances, purposes, or cross-task generalization. In contrast, our work targets procedural knowledge reasoning, where models must integrate observed video evidence with external structured knowledge to answer complex, multi-hop questions.

\subsection{Knowledge-based Visual Question Answering} 

Knowledge-based visual question answering has become popular in recent years~\cite{marino2019ok,chang2022webqa,park2020visualcomet,schwenk2022okvqa}. The majority of the past benchmarks on knowledge-based visual question answering are for images while more recently authors in ~\cite{wang2024sok} presented a benchmark to evaluate situated and open-world common-sense reasoning in videos with compositionality, temporality, and causality. While the motivation of our work shares similarities with~\cite{wang2024sok}, there are notable distinctions. Their focus lies in situated commonsense reasoning grounded in the specific contexts depicted in videos, whereas we concentrate on testing the procedural understanding of models given some partial information such as a step of a task.

\subsection{Neuro-Symbolic and Modular Reasoning}

Neuro-symbolic reasoning frameworks aim to bridge neural perception with symbolic structure~\cite{de2020statistical,abbe2022learning}. Modular networks, neural state machines~\cite{hudson2019learning}, and program-based reasoning models enable compositional reasoning by assembling neural components according to predicted programs~\cite{johnson2017inferring,mascharka2018transparency,andreas2016learning,perez2018film,hudson2019learning,hudson2018compositional,chen2021meta,endo2023motion,li2025imore}. More recent systems, such as ViperGPT~\cite{suris2023vipergpt} and related approaches, leverage LLMs to generate executable programs that invoke visual or symbolic functions.
Despite their promise, existing modular systems typically rely on predefined, hand-engineered functions or operate over limited sets of symbolic operators. They also lack mechanisms to learn domain-specific relational knowledge directly from data. In contrast, KML learns relation-specific neural modules directly from a procedural knowledge graph, allowing it to scale to rich procedural domains while maintaining interpretability.

\subsection{Knowledge Graph Embeddings and Reasoning}
Knowledge graph embedding methods such as TransE~\cite{bordes2013translating}, TransH~\cite{wang2014knowledge}, and RotatE~\cite{sun2019rotate} represent entities and relations in continuous vector spaces and support link prediction and multi-hop reasoning through embedding composition. 
Mostly, these methods only learn the embeddings using some geometric properties and assumptions of the vector spaces that they operate.
In contrast KML learns neural networks to represent relations as mapping functions, and entities as embeddings as that operates through relations as mapping functions.
The closest idea to our work is the pioneering work of NTN~\cite{socher2013reasoning}. However, NTN is only able to predict a confidence score for a given triplet and multi-hop reasoning becomes challenging with it's use of complex bilinear neural network models to encode relations.
Furthermore, these early methods are effective for knowledge completion, yet they are not designed to integrate perceptual grounding or to provide interpretable intermediate reasoning steps.

KML differs fundamentally from KG embedding approaches by treating relations as neural mapping functions rather than algebraic operators, enabling uncertainty-aware reasoning and seamless integration with multimodal grounding. Furthermore, KML’s execution traces align with symbolic reasoning paths, offering interpretability that embedding-based approaches lack.

\subsection{LLM-Based Reasoning and Program Synthesis}
LLMs have demonstrated strong reasoning capabilities through prompting strategies such as chain-of-thought~\cite{wei2022chain}, tree-of-thought~\cite{yao2023tree}, and graph-of-thought~\cite{besta2024graph}. These methods improve reasoning transparency but do not enforce semantic constraints on intermediate steps, often leading to hallucinations or inconsistent reasoning in structured domains.

Large VLMs such as Flamingo~\cite{alayrac2022flamingo}, Video-ChatGPT~\cite{maaz2023video}) and \textbf{neurosymbolic} frameworks like ViperGPT~\cite{suris2023vipergpt}, MoreVqa~\cite{min2024morevqa} and~\cite{choudhury2024video} decouple reasoning (e.g., program generation) from execution (e.g., API calls). While ViperGPT uses predefined functions for visual queries, it lacks mechanisms to learn domain-specific knowledge modules or constrain reasoning to procedural logic. 

KML leverages LLMs for program synthesis, but constrains execution to a predefined set of learned knowledge modules aligned with a procedural schema. This decoupling allows KML to benefit from the reasoning flexibility of LLMs while ensuring correctness, interpretability, and alignment with domain knowledge.

\subsection{Integrating External Knowledge with Video Understanding}
External Knowledge has been explored through knowledge graphs (KGs) for tasks like activity recognition~\cite{ma2022visual,ghosh2020all} and visual commonsense reasoning~\cite{lin2019kagnet}. 
VidSitu~\cite{sadhu2021visual} links events to semantic roles, while TVQA+\cite{lei2019tvqa+} uses script knowledge for story-based QA. However, these works focus on descriptive reasoning rather than procedural reasoning.

\subsection{Logical queries}
Our work is also related to NLP research such as~\cite{abdelaziz2021semantic,yang2023llm,zhong2024synthet2c,shah-etal-2024-improving}, which employ \textbf{logical queries} to extract information from a KG. This emphasis on curated, domain-specific knowledge is especially valuable for AI assistants operating in contexts where domain expertise, rather than general common-sense reasoning, is crucial for success. 

\section{Procedural Knowledge Graph Construction}
Motivated by recent advances in the semi-automated construction of knowledge-based question answering datasets \cite{hoang2024semi}, we introduce the \textit{\dlong} (\dataset) dataset. 
It is built on a \textit{Procedural Knowledge Graph} (\pkg) that integrates information from the COIN training set~\cite{tang2019coin}, the COIN ontology~\cite{tang2019coin}, GPT-4o-generated annotations, and external common-sense knowledge from ConceptNet~\cite{speer2017conceptnet}, followed by human verification.
To enable systematic QA generation, we define a set of question templates and associated Cypher queries~\cite{francis2018cypher}, which are executed over the knowledge graph to retrieve correct answers. In the following sections, we detail the construction of both the knowledge graph and the QA dataset.

\subsection{Defining \pkg's Schema (\kgs).}
A knowledge graph schema provides a high-level abstraction of the graph, specifying the types of entities that exist and the valid relationships that can occur between them. It serves as a blueprint for structuring the data and interpreting the semantics of the graph. In our case, \kgs contains nodes where each node correspond to a entity type ($\mathcal{E}$) and edges represent relation types ($\mathcal{R}$) (\Cref{fig:dataset_creation}). 
The core entity types in \kgs include \textit{Domain}, \textit{Task}, \textit{Step}, \textit{Action}, \textit{Object}, \textit{Tool}, and \textit{Purpose}. Relation types capture meaningful procedural links, such as task-step associations or tool usage (e.g. \texttt{HAS\_TOOL}), and may carry attributes like \textit{id}, \textit{type}, or additional semantic metadata.
The detailed descriptions and definitions of each entity and relation is given in~\Cref{sec.detailed.pkg}.

\subsection{Populating \pkg.}
Based on this schema, we instantiate \pkg by populating it with specific entity and relation instances. Each entity $e \in \mathcal{E}$ has three main attributes: \textit{type}, \textit{name}, and \textit{id}. Relations $r \in \mathcal{R}$ represent specific connections between entity instances. To construct this graph, we integrate annotations from the COIN dataset~\cite{tang2019coin}, fine-grained information extracted using GPT-4o~\cite{hurst2024gpt}, and commonsense knowledge from ConceptNet~\cite{speer2017conceptnet}. 

\textbf{COIN-based Data} (\texttt{Domain, Task, Step}):
We extract the procedural structure from the training split of the COIN dataset, which provides annotations for \textit{domains}, \textit{tasks}, and \textit{steps}. For each unique instance of these entities, we create a node in the graph. We define \texttt{HAS\_TASK} edges to link each domain to its associated tasks, and \texttt{HAS\_STEP} edges to connect tasks to their constituent steps. To capture procedural flow, we analyze step sequences from all training videos and construct \texttt{HAS\_NEXT\_STEP} relations that represent the temporal ordering of steps. We also introduce \texttt{START} and \texttt{END} nodes to explicitly mark task boundaries. Each \texttt{HAS\_NEXT\_STEP} edge is annotated with its observed frequency to model empirical transition likelihoods between steps.

\textbf{LLM-Augmented Data} (\texttt{Action, Object, Tool}):
To enhance procedural specificity, we use GPT-4o\footnote{\url{https://openai.com/index/hello-gpt-4o/}} to extract action-object pairs from step descriptions. For instance, the phrase ``remove the tire'' yields the action ``remove'' and the object ``tire''. Because COIN lacks explicit tool annotations, we further prompt GPT-4o to infer potential tools for each step. 
The prompt that we use is shown in~\Cref{fig:extractingprompt} (Appendix).
The tool lists are manually verified. We standardize terms (e.g., ``scissor'' $\rightarrow$ ``scissors''), consolidate duplicates (e.g., ``marker'', ``marker pen''), and unify task-specific variants (e.g., ``toilet detergent'' $\rightarrow$ ``detergent''). Some task-sensitive distinctions are preserved (e.g., ``filter'' $\rightarrow$ ``coffee filter'') when relevant to the procedural context.
 
%

\textbf{ConceptNet-based Data} (\texttt{Purpose}):
In procedural tasks, actions are typically performed—and tools or objects used—with specific purposes. To represent such intent in our graph, we incorporate commonsense knowledge from ConceptNet, a large-scale semantic network that connects words and phrases through meaningful, human-readable relations such as \texttt{UsedFor}, \texttt{CapableOf}, and \texttt{IsA}. Specifically, we extract potential purposes of actions, tools, and objects using the \texttt{UsedFor} and \texttt{CapableOf} relations.
To contextualize these purposes within specific procedural steps and tasks, we use GPT-4o to infer task- and step-specific interpretations for each entity’s purpose. This step ensures that the resulting knowledge is not generic but grounded in the procedural context in which the entity is used. (See ~\Cref{fig:purpose_grounding} -- Appendix)

To reduce redundancy and merge semantically similar purposes, we compute pairwise cosine similarities of Sentence-BERT embeddings \cite{reimers-2019-sentence-bert} and apply a similarity threshold of 0.8, validated through manual inspection. Pairs of highly similar purposes are linked using a \texttt{HAS\_SIMILAR\_PURPOSE} edge. The resulting purpose-augmented graph is stored in Neo4j\footnote{\url{https://neo4j.com/}}, supporting efficient querying and reasoning via Cypher. All contents are manually verified to ensure the reliability of the KG. The constructed KG contains 2,954 unique entities and 12,484 relations. 
The distributions of entities and relations are shown in Figure~\ref{fig:kg_statistic} of the supplementary.


\subsection{Detailed definition of \pkg}
\label{sec.detailed.pkg}

\Cref{fig:dataset_creation}(left) shows \kgs, the schema of our procedural knowledge graph, \kg.
\kgs consists of 9 nodes (entities, $\mathcal{E}$) and 14 edges (relations, $\mathcal{R}$) connecting these nodes.
The nodes represent various types of entities, each capturing a key aspect of procedural knowledge, including:
\begin{itemize}[leftmargin=0pt]
\setlength\itemsep{0em}

    \item[] \textbf{Domain}: contains 12 distinct domains which represent the highest-level category or field under which a set of related tasks are grouped.

    \item[] \textbf{Task}:  represents a specific activity or goal within a domain that requires a set of coordinated steps to complete. Tasks provide the contextual framework for steps, defining the objective that the steps aim to achieve. It bridge the gap between the high-level domain and the detailed procedural steps.
    
    \item[] \textbf{Step}: detailing the individual operations necessary to accomplish a specific task. Steps are sequentially organized to capture the logical flow required for successful task completion.

    \item[] \textbf{Action}: represents atomic operation performed as part of a step. Actions are more granular levels that detail how a step is executed. For example, actions such as ``\textit{scrub}'', ``\textit{fry}'', or ``\textit{pump up}'' describe the specific motions or manipulations required in each step. 


    \item[] \textbf{Tool}: represents a specific tool that can be used to assist in performing a step. These tools are not just generic but are tied to the visual evidence in the procedural context. For example, a ``\textit{sponge}'' and ``\textit{cleaning solution}'' for the step ``\textit{scrub the bathtub}'' under task ``\textit{clean bathtub}'' are tools. 

    \item[] \textbf{Object}: refers to the item being manipulated or acted upon in a step. Objects are central to the procedural actions as they undergo changes or are used as part of the process. Using previous example, in the step ``\textit{scrub the bathtub}'', the object is the bathtub.

    \item[] \textbf{Purpose}: captures the intention or function behind a specific grounded tool within the context of a task and step. It is important for answers the ``\textit{what}'' tool can be used in the step if we want to achieve a certain task. For example, the ``\textit{sponge}'' has purpose of ``\textit{cleaning}'' under step ``\textit{scrub the bathtub}''. This node type enhances the knowledge graph's interpretability by explicitly associating tools with their roles in achieving procedural goals.

    \item[] \textbf{START} and \textbf{END}: two dummy nodes that connect to the starting step and ending step, respectively. These two entities are used to identify starting step and ending step of a task. 
\end{itemize}

The edges in \kgs capture interactions, dependencies, and hierarchies between entities. These relations include:

\begin{itemize}[leftmargin=0pt]
\setlength\itemsep{0em}

    \item[] \textbf{HAS\_TASK}: links a \textbf{Domain} node to its associated \textbf{Task} nodes. It represents the grouping of tasks under a domain.
    
    \item[] \textbf{HAS\_STEP}: connects a \textbf{Task} node to its associated \textbf{Step} nodes, capturing the breakdown of the task.
    
    \item[] \textbf{HAS\_NEXT\_STEP}: capture the sequential order between \textbf{Step} nodes within a task. For example, in the task ``\textit{clean bathtub}'', the step ``\textit{scrub the bathtub}'' is followed by the step ``\textit{clean bathtub with water}''. This relation helps in modeling the progression of actions. 
    
    \item[] \textbf{HAS\_ACTION}: connects a \textbf{Step} node to its corresponding \textbf{Action} nodes, detailing the atomic operations required to execute the step.
    
    \item[] \textbf{HAS\_TOOL}: links a \textbf{Step} node to the \textbf{Tool} nodes that assist in performing the step.
    
    \item[] \textbf{HAS\_OBJECT}: links a \textbf{Step} node to the \textbf{Object} nodes that are manipulated or acted upon during the step.
    
    \item[] \textbf{HAS\_PURPOSE}: connects a \textbf{GroundedTool} node to a \textbf{Purpose} node, capturing the function or intent of the tool in the procedural context.
    
    \item[] \textbf{HAS\_SIMILAR\_PURPOSE}:  links two \textbf{Purpose} nodes that share similar intents or functions. For instance, the purposes ``\textit{carry liquids}'' and ``\textit{hold liquids}'' are linked with this relation.

\end{itemize}

The structure of our PKGS not only captures the hierarchical structure and temporal relationships between various components of procedural knowledge extracted from \cite{tang2019coin}, but also extends its expressiveness and utility by incorporating additional layers of information. By enriching the graph with detailed node types and their interconnections, our PKGS provides a more holistic representation of procedural knowledge. 

\section{Question Answering Dataset Creation}
\dataset is a multiple-choice question answering dataset designed to evaluate reasoning over procedural knowledge in instructional videos. Each instance consists of a video segment, a question, five answer choices, and one correct answer. Unlike existing video QA datasets that focus on grounding answers directly in the visual content of a single video, \dataset emphasizes reasoning over procedural knowledge that extends beyond the given video.

\subsection{Traversal Templates for Procedural Reasoning}

We define questions in \dataset based on traversal templates which are specific reasoning patterns over \pkg. For instance, the traversal
$\texttt{Step} \xrightarrow{\texttt{HAS\_TOOL}} \texttt{Tool} \xrightarrow{\texttt{HAS\_PURPOSE}} \texttt{Purpose}$
corresponds to the question: “What is the purpose of the tool used in this step?”. We design 17 such traversal templates to cover a wide range of procedural reasoning types. For each template, we generate multiple question variants using GPT-4o-mini, followed by manual filtering to ensure quality and clarity. These templates form the backbone for producing diverse yet semantically consistent question-answer instances.
~\Cref{fig:dataset_creation} presents examples of Traversal Templates alongside their corresponding questions. 
The complete set of 17 question templates with example questions is provided in the \Cref{tab:dataset} - Appendix.

\subsection{Question and Answer Generation}
We generate questions by aligning each video segment with its corresponding \texttt{Task} and \texttt{Step} nodes in \pkg, then apply traversal templates to form questions and retrieve answers. For each question, we sample one correct answer and four distractors. 
When generating answer options for each question, relying solely on random sampling can lead to biased options. For instance, frequent terms (e.g., tasks, grounded tools, or objects) are more likely to appear in the correct options compared to infrequent terms, which tend to appear only in the incorrect options. This creates a shortcut for models to identify the correct answer without engaging in true information retrieval and reasoning. To address this issue, we design tailored sampling strategies for each traversal template, ensuring a relative-balanced distribution of terms between correct and incorrect answer options.

Additionally, each question is paired with a Cypher query that retrieves supporting facts from \pkg. These structured annotations serve as logical forms for KG-based evaluation, model supervision, and reasoning trace analysis.

\dataset is designed for scenarios with limited training data but having access to structured knowledge, such as a knowledge graph. We construct a training set of 1,700 samples (100 per traversal template) and a validation set of 850 samples (50 per template). The test set contains 46,921 questions, generated from all video segments in the COIN test split. 
This setup supports zero-shot and few-shot generalization. 

To assess the quality of our dataset, we evaluate whether a question can be reasonably answered by a human—referred to as \textit{plausibility}. We conducted a case study with eight participants, each answering a subset of 170 questions. For each question, participants were shown a video segment, the question, and five answer options, and were asked to select the correct answer. Each question was independently answered by three participants. 
A question is considered plausible if at least one participant selected the correct answer. Using this criterion, we found that 92.4\% of the questions are plausible. 
Additionally, random baselines perform at chance level ($\sim 20$\% accuracy), suggesting low annotation artifacts or answer biases.

\subsection{Rephrasing Question Templates}

To ensure a diverse range of question phrasings in our dataset, we performed a rephrasing process. Specifically, we prompted GPT-4o-mini to generate up to 30 semantically diverse and contextually appropriate rephrasing for each traversal-based natural language question template. This step leverages the LLM’s capability to enhance linguistic variety while preserving meaning. To ensure quality and consistency, all rephrased questions were manually reviewed to confirm that the original intent was retained. The prompt is shown in in~\Cref{fig:prompt.rephrasing} - Appendix.

\section{Knowledge Module Learning}
\label{sec.module}
In this section, we introduce our \textit{Knowledge Module Learning} (\kml) for procedural knowledge reasoning, which effectively handles uncertain (probabilistic) inputs while producing interpretable reasoning outputs using a small number of trainable parameters. 
We train a collection of neural networks known as Neural Knowledge Module (\textbf{KM}) to represent each binary relation type of \pkg. 
Then given a video and a question, we ask LLM to generate a program consisting of KM invocations to answer the questions. Then we execute those KMs sequentially with the grounded evidence extracted from the video to obtain the final answer. 
\Cref{fig:model} shows the overall of our Knowledge Module Learning (KML) framework.
Next, we present the details. 

\begin{figure}[t]
   \centering
   \includegraphics[width=.9\linewidth]{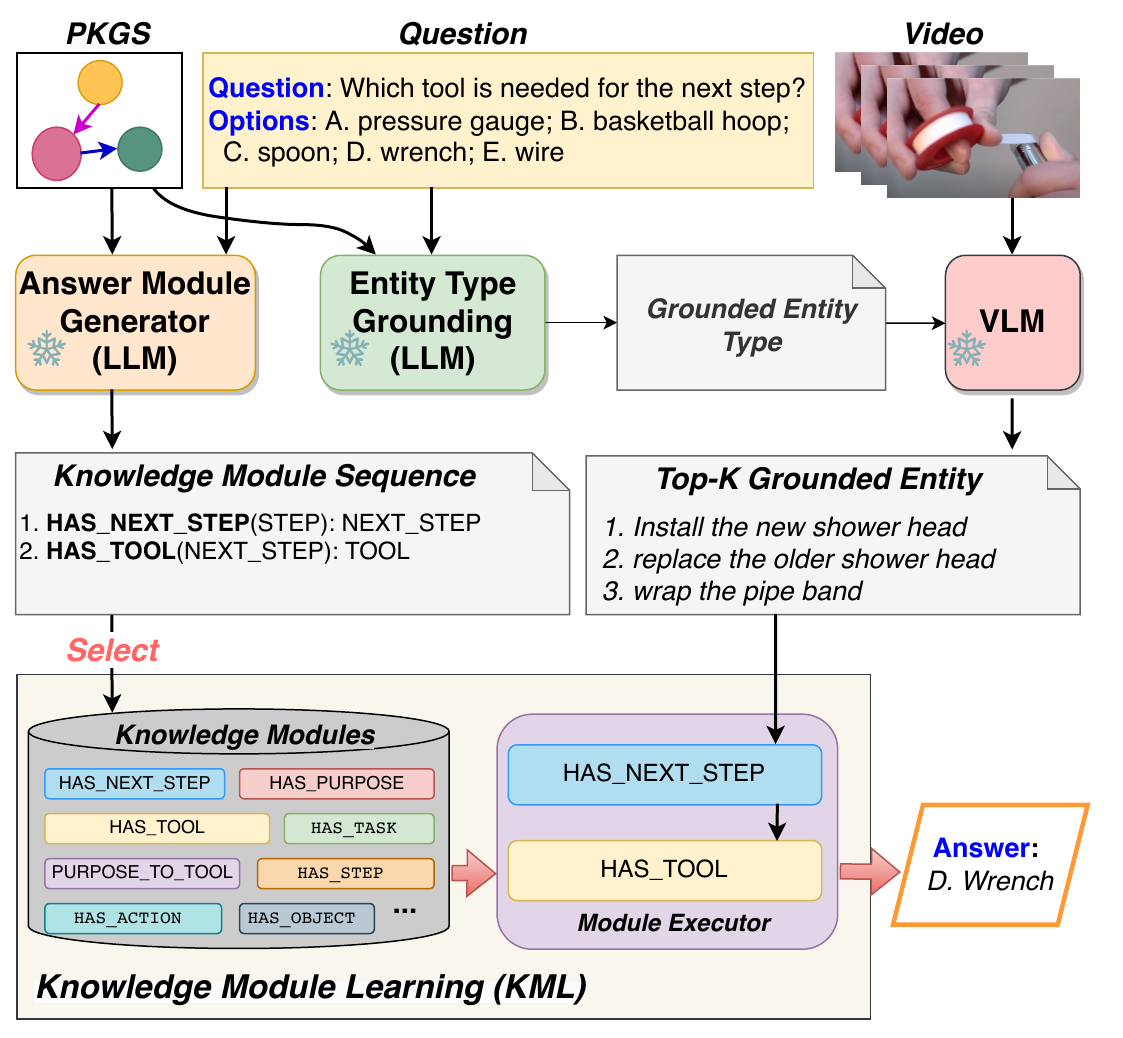}
   \caption{Knowledge Module Learning (\skml) for knowledge graph-based procedural video question answering. Given a video and a question, the framework first identifies the entity type in \pkg (e.g. ``step'')  that should be detected in the video and then recognizes the entity type instances (e.g., step: ``\textit{wrap the pipe band}'') using a VLM. Based on the question and grounded entity type, the \textit{answer module generator} produces the module sequence. Finally, \skml executes the sequence of Knowledge Modules to answer the question. In this example, it first identifies the next step and then determines the tool required for that step, which is a ``\textit{wrench}''.}
\label{fig:model}
\end{figure}

\subsection{Knowledge Module Learning}
For each binary relation type $\mathcal{R}_k(\mathcal{E}_{i}, \mathcal{E}_{j})$ in the \pkg that maps from entity type $\mathcal{E}_{i}$ to $\mathcal{E}_{j}$, we learn a KM as a learnable neural relation ($\phi_{{\mathcal{R}}_{k}}$) with parameters ($\theta_{{\mathcal{R}}_{k}}$).
KM learns to map from a given head entity $e_i$ to the corresponding set of tail entities $\mathcal{I}_{j}$ where $\mathcal{I}_{j} \subseteq \mathcal{E}$ of relation $\mathcal{R}_k$ as follows:
\begin{equation}
     \phi_{{R}_{k}} (~h(e_{i};~\theta_h);~\theta_{{R}_{k}}) \rightarrow \mathbf{e_j}   
     \label{eq.mapping}
\end{equation}
where $h(e_{i};\theta_h)$ is a d-dimensional embedding function. 

The entity $e_{i}$ is mapped to $\mathbf{e_j}$, a vector that is closer to $x_j = h(e_{j})$ for all $e_j \in \mathcal{I}_{j}$ (the corresponding set of tail entities).
Formally, $\mathcal{I}_{j}$ is the set of tail entities under relation $R_k$ for the head entity $e_{i}$,~i.e.~$\mathcal{I}_{j} := \left\{ e_j \in \mathcal{E} \;\middle|\; R_k(e_i, e_j) = True \right\}$.

The objective of Knowledge Module Learning (KML) is to learn a relation-specific neural mapping
function $\phi_{R_k}$ (cf.~\Cref{eq.mapping}) such that, given a head entity $e_i$, the mapped
representation $\phi_{R_k}(h(e_i))$ is closer (under cosine similarity) to embeddings of all
valid tail entities than to embeddings of invalid entities.

Let $ X = \{x_1, x_2, \dots, x_n\} \subset \mathbb{R}^D$
denote the set of learned entity embeddings ($x_j = h(e_j, \theta_h)$) in the knowledge graph, where each $x \in X$ is $\ell_2$-normalized, i.e., $\|x(e)\|_2 = 1$.
For a fixed relation $R_k$ and a head entity embedding $x \in X$, we define
$Y(x) \subseteq X$: the set of \emph{positive tail entity embeddings} satisfying    $(e_i, R_k, e_j)$.
Let $U(x) := X \setminus Y(x)$: the set of \emph{negative} tail entity embeddings and
$z := \phi_{R_k}(x)$: the output of the relation-specific knowledge module where
$\hat z := z / \|z\|_2$ is the normalized module output. We define cosine similarity as $
s(a,b) := \langle a, b \rangle, \qquad \|a\|_2 = \|b\|_2 = 1,$
so that $s(a,b) \in [-1,1]$.

We iterate over all the triplets (including the inverse triplets) of all relation types of \pkg in a batch-learning manner and train the KMs using the contrastive loss.

\begin{equation}
\mathcal{L}(x)
\;=\;
- \log
\frac{\sum\limits_{y \in Y(x)} \exp\!\big(\frac{s(\hat z, y)}{\tau}\big)}
{\sum\limits_{v \in X} \exp\!\big(\frac{s(\hat z, v)}{\tau}\big)}.
\label{eq:multi_pos_loss}
\end{equation}

Here $\tau$ is the temperature.
Contrastive loss plays a crucial role in learning neural relation functions $\phi_{{R}_{k}}$ (i.e. KMs), for modelling symbolic binary relations of the form $\mathcal{R}_k(\mathcal{E}_{i}, \mathcal{E}_{j})$ as later shown by the theoretical analysis on the learning of such mappings in~\Cref{sec:theory.one}.

The embedding learning function $h(;\theta_h)$ is implemented using CLIP~\cite{radford2021learning} text encoder embeddings with frozen parameters ($\theta_h$).
Alternatively, we also learn the embedding function from scratch using standard implementations\footnote{We learn these using torch.nn.Embedding}.
Similarly, we also learn the inverse KM for each inverse relation ${\mathcal{R}^{*}}_{k} ( \mathcal{E}_{j}, \mathcal{E}_{i} )$ of each relation $\mathcal{R}_{k}$.
Let us denote the set of all KMs by $\phi_{R} = \{ \phi_{{R}_{k}} | k = 1, \cdots \}$.


At inference, KM takes an input embedding and maps that to an output embedding that represents a set of corresponding tail entities of that relation.  For example, \Cref{fig:q1} shows the semantic meaning of each output embedding that maps to a set of tail entities of relation \texttt{HAS\_TOOL} for given input embedding of the Step entity.
We experimented with different neural configurations of multi-layered perceptron (MLP) and found that two-layered MLP with \textit{Tanh} activation performs the best for learning KMs.
As shown later in~\Cref{sec:det_comp_bound}, controlling the Lipschitz constant of the neural mapping function in KMs is crucial. Therefore, both the choice of activation functions and the design of the MLP are carefully considered.
Next, we present how to answer questions using KMs and LLM generated programs.

\subsection{Question Answering suing KML}

Given the video $v$, the question $Q$, and options $O=\{o_1, \cdots, o_n \}$ ($n=5$) we  prompt a LLM $\phi()$ to find the relevant entity type ($E_g$) that should be grounded in the video to answer the question.
\begin{equation}
    \phi(Q, \text{\kgs})\rightarrow E_g 
    \label{eq.eg}
\end{equation}
For example, $E_g$ can be a task, step, object or action.
Then we find the entity instance(s) that is present in the given video $V$ of the entity type 
$E_g$ using a vision foundation model (VLM) as follows:
\begin{equation}
    \text{VLM}(V, E_g)\rightarrow \mathbf{x_e} 
\end{equation}
where $\mathbf{x_e} \in \mathcal{R^{C}}$ represents the score vector for the entity type (e.g., step distribution) across all categories of that entity type. We assume there are $C$ categories for entity type $E_g$. 
To obtain estimates about the grounding entity $E_g$, we use \emph{ProceduralVRL} (P.VRL)~\cite{zhong2023learning} a VLM tailored for procedural tasks.

Given the collection of Knowledge Module names ($\phi_{R}$) and the question $Q$, and the grounded entity type $E_g$, we use LLM  (denoted by $\phi_G()$) to generate a list of sequence of Knowledge module invocations (known as program or $P_1$) to answer the question. 
Using a similar concept to chain-of-thought~\cite{wei2022chain,wang2022self}, tree-of-thought~\cite{yao2023tree}, and graph-of-thought~\cite{besta2024graph} we ask the LLM to generate multiple alternative programs to answer the same questions.
\begin{align}
\label{eq.getprogram}
    \phi_G(Q, \phi_{R} , E_g) = [P_1=\left< \phi_{{r}_{i}}, \phi_{{r}_{j}}, \phi_{{r}_{k}}, ... \right >, \cdot, & \\ \nonumber P_l = \left< \phi_{{r}_{i}}, \phi_{{r}_{j}}, \phi_{{r}_{k}}, ... \right >] 
\end{align}
Here $P_l$ is the specific module invocations and each $P_l$ may result in different answers following an alternative-thought of reasoning approach.
The LLM module $\phi_G()$ can be implemented with any LLM with good reasoning capability. 
We use a special prompt that invokes deep understanding and consistent knowledge graph traversal and alternative thought invocations.
We use LLM to model the ~\Cref{eq.eg} and ~\Cref{eq.getprogram} (See~\Cref{fig:programprompt} - Appendix).
%

%
As an example, to answer the question \texttt{what is an alternative tool that can be used for this step?} it generates the  program shown in~\Cref{fig:llm.progra.example}.
\begin{figure}[t!]
\centering    
\includegraphics[width=0.99\linewidth]{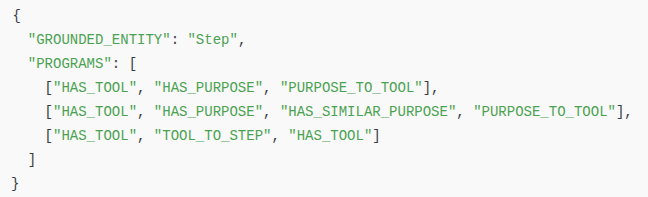}
\caption{Example of logical programs generated by LLM. 
}
\label{fig:llm.progra.example}
\end{figure}
Here PURPOSE\_TO\_TOOL(Purpose) is the inverse relation of HAS\_PURPOSE(Tool).
LLM may generate multiple programs that leads to the answer, typically 3-5 alternatives for complex problems.
To predict the answer, we execute each program in order, inputting the grounded entity representation $z_i$ into the first module. For each program ($P_1$ to $P_l$), $z_i$ is fed into the first module to obtain intermediate embedding, invoking all modules sequentially.
For example $z_j = \phi_{{r}_{i}}(z_i)$ then $ z_k = \phi_{{r}_{j}}(z_j)$ and finally $z_f = \phi_{{r}_{k}}(z_k)$.
Therefore the \emph{final embedding} representing the answer to the given question is $z_f$.
We also inspect the meaning of each intermediate representation as shown in~\Cref{fig:q1} allowing more interpretable reasoning that can handle uncertain inputs.

To compute the first input embedding $z_i$, we use the grounded top-K entity instance and weight the embeddings of each grounded entity instance category name as follows:
\begin{align}
    z_i &= S \times {\mathbf{X}_e}^{T} \\
    {\mathbf{X}_e} &= [h(e_1), h(e_2), \cdots, h(e_k)]
\end{align}
where $e_k$ is the $k$-th category of the entity type $E_g$ and $S=[s_1, s_2, \cdots, s_k]$ is a vector containing top-k scores for each grounded entity-type category (e.g. Step categories).

\noindent
\textbf{Inference}
Given the options $O=\{o_1, \dots , o_n\}$, we obtain the embeddings of $O$, i.e., 
\begin{equation}
    h(O) = [h(o_1), h(o_2), \cdots h(o_n)]. 
\end{equation}

Then we compute the cosine similarity between the \emph{final embedding} $z_f$ and option embeddings $h(O)$ and apply softmax to predict the answer index for each program.
When there are alternative programs, we take the maximum score from all programs as the final answer.

\subsection{VQA Training for KML.}
We also fine-tune the KMs using a few examples of the video-question-answer using our dataset.
We compute the cosine similarity between $z_f$ and $h(O)$ and apply softmax to predict the answer scores.
\begin{equation}
    \hat{y}= softmax\left( cosine\_sim\left( h(O),  z_f\right)\right)
\end{equation}
Then we train all KMs jointly using the cross-entropy loss over answer score distribution where the correct answer is used as the ground truth label for cross-entropy training.
At test time, we predict the right answer using the argmax of the cosine similarity scores.
One of the advantages of KML (auto-program) approach is that, given the relation types, we do not need to manually select any neural modules, or generate programs manually. We use LLMs such as GPT, DeepSeek or Mistral.

\section{Theory}

\subsection{Theoretical Analysis of \textbf{KML} Learning}
\label{sec:theory.one}


In this section, we provide a formal analysis of learning objective of KMs and derive conditions
under which the learned mapping satisfies the desired separation property.
The ideal outcome of KML learning for relation $R_k$ is that the mapped representation
$\hat z = \phi_{R_k}(x)$ separates positives from negatives:
\begin{equation}
\min_{y \in Y(x)} s(\hat z, y)
\;>\;
\max_{u \in U(x)} s(\hat z, u).
\label{eq:kml_separation}
\end{equation}
This condition ensures that even the least similar positive tail entity is closer to $\hat z$
than the most similar negative entity.

While KML is trained with contrastive loss as in~\Cref{eq:multi_pos_loss}.
For clarity in this theoretical analysis, we omit the temperature parameter, as it does not affect the ordering arguments.

Let us define
\[
A := \sum_{y \in Y(x)} e^{s_y},
\quad
C := \sum_{u \in U(x)} e^{s_u},
\quad
B := A + C,
\]
where $s_v := s(\hat z, v)$.

\begin{lemma}[Loss bound implies relative mass constraint]
\label{lem:mass_bound}
If $\mathcal{L}(x) \le \varepsilon$, then
\[
C \;\le\; (e^{\varepsilon} - 1)\,A.
\]
\end{lemma}

\begin{proof}
From the definition of $\mathcal{L}(x)$,
\[
\mathcal{L}(x) \le \varepsilon
\;\;\Longleftrightarrow\;\;
\frac{A}{A+C} \ge e^{-\varepsilon}.
\]
Rearranging yields
\[
C \le (e^{\varepsilon}-1)\,A.
\]
\end{proof}

\begin{lemma}[Bounds on extreme similarities]
\label{lem:extreme_bounds}
The following inequalities hold:
\[
\min_{y \in Y(x)} s_y \;\ge\; \log A - \log |Y(x)|,
\qquad
\max_{u \in U(x)} s_u \;\le\; \log C.
\]
\end{lemma}

\begin{proof}
For the positive set,
\[
A = \sum_{y \in Y(x)} e^{s_y}
\;\ge\;
|Y(x)|\, e^{\min_{y \in Y(x)} s_y},
\]
which implies
\[
\min_{y \in Y(x)} s_y \ge \log A - \log |Y(x)|.
\]
For the negative set,
\[
C = \sum_{u \in U(x)} e^{s_u}
\;\ge\;
e^{\max_{u \in U(x)} s_u},
\]
which implies
\[
\max_{u \in U(x)} s_u \le \log C.
\]
\end{proof}

\begin{theorem}[Sufficient condition for KML separation]
\label{thm:kml_separation}
If
\begin{equation}
\mathcal{L}(x)
\;<\;
\log\!\left(1 + \frac{1}{|Y(x)|}\right),
\label{eq:kml_loss_bound}
\end{equation}
then the separation condition in \Cref{eq:kml_separation} holds.
\end{theorem}

\begin{proof}
From Lemmas~\ref{lem:mass_bound} and~\ref{lem:extreme_bounds},
\[
\min_{y \in Y(x)} s_y - \max_{u \in U(x)} s_u
\;\ge\;
- \log |Y(x)| - \log(e^{\varepsilon}-1),
\]
where $\varepsilon = \mathcal{L}(x)$.
The right-hand side is positive if and only if
\[
|Y(x)|\,(e^{\varepsilon}-1) < 1,
\]
which is equivalent to
\[
\varepsilon < \log\!\left(1 + \frac{1}{|Y(x)|}\right).
\]
Under this condition,
\[
\min_{y \in Y(x)} s_y > \max_{u \in U(x)} s_u,
\]
and therefore the KML separation condition holds.
\end{proof}

\paragraph{Interpretation.}
Theorem~\ref{thm:kml_separation} establishes a concrete link between the value of the KML
contrastive loss and the correctness of relational reasoning.
When the loss falls below the bound in \Cref{eq:kml_loss_bound}, the learned knowledge module
is guaranteed to separate valid tail entities from invalid ones under cosine similarity.
This provides an explicit criterion indicating when KML learning is expected to succeed
and when separation may fail, offering theoretical insight into the behavior of the model
during training.


\subsection{Deterministic Composition Error Bound over Traversals}
\label{sec:det_comp_bound}

In KML, a reasoning program (or traversal) is a sequence of relation modules.
We provide a deterministic error bound on the deviation between the output of the learned
Knowledge Modules (KMs) and the output of an ideal (ground-truth) relational operator sequence. The bound explicitly shows how errors accumulate across hops and how Lipschitz continuity controls their propagation. A length-$T$ program is a sequence of relations
\[
P = (r_1, r_2, \ldots, r_T).
\]
For each relation $r_t$, KML learns a module $\phi_{r_t}:\mathbb{R}^d\rightarrow\mathbb{R}^d$ (a small MLP in our implementation).
To analyse approximation, we introduce an \emph{ideal} operator $F_{r_t}:\mathbb{R}^d\rightarrow\mathbb{R}^d$
that represents the desired mapping induced by the KG relation semantics (e.g., mapping a head-entity representation to a canonical representation of the valid tail-set for that relation).

Given an input state $z_0$ (the grounded entity representation), the ideal and learned executions are
\begin{align}
\label{eq:ideal_exec}
z_t &= F_{r_t}(z_{t-1}), \qquad t=1,\ldots,T,\\
\label{eq:learned_exec}
\hat z_t &= \phi_{r_t}(\hat z_{t-1}), \qquad t=1,\ldots,T,
\end{align}
where $z_t$ is the ideal state after $t$ hops and $\hat z_t$ is the KML-computed state.
We define the per-hop (deterministic) deviation using L2 loss as follows:
\begin{equation}
\label{eq:delta_def}
\delta_t := \|\hat z_t - z_t\|_2.
\end{equation}

We quantify how well each learned module $\phi_{r_t}$ approximates its ideal operator $F_{r_t}$ on the unit sphere:
\begin{equation}
\label{eq:eps_def}
\varepsilon_t := \sup_{\|u\|_2=1} \|\phi_{r_t}(u) - F_{r_t}(u)\|_2.
\end{equation}
Intuitively, $\varepsilon_t$ is the \emph{worst-case} (deterministic) approximation error of module $r_t$
over all unit-norm inputs.

We know each \emph{learned} module is Lipschitz continuos on the unit sphere by design:
\begin{equation}
\label{eq:lipschitz_phi}
\|\phi_{r_t}(u) - \phi_{r_t}(v)\|_2 \le L_t \|u - v\|_2,
\quad \forall \|u\|_2=\|v\|_2=1,
\end{equation}
for some constant $L_t\ge 0$ that may depend on the relation $r_t$.
This is a standard and verifiable property for MLPs: if each layer has bounded spectral norm and the activation
is Lipschitz (e.g., $\tanh$ is 1-Lipschitz), then the whole network is Lipschitz with constant bounded by
the product of layer norms.

\bigskip
\begin{lemma}[One-step error recursion under Lipschitz learned module).]
\emph{
For each hop $t$,
\begin{equation}
\label{eq:one_step_recursion}
\delta_t \le L_t\,\delta_{t-1} + \varepsilon_t.
\end{equation}
}
\end{lemma}

\paragraph{Proof.}
We start from the definition of $\delta_t$ and substitute the executions \eqref{eq:ideal_exec}--\eqref{eq:learned_exec}:
\begin{align*}
\delta_t
&= \|\hat z_t - z_t\|_2 \\
&= \|\phi_{r_t}(\hat z_{t-1}) - F_{r_t}(z_{t-1})\|_2.
\end{align*}
We now add and subtract $F_{r_t}(\hat z_{t-1})$ inside the norm:
\[
\phi_{r_t}(\hat z_{t-1}) - F_{r_t}(z_{t-1})
=
\underbrace{\big(\phi_{r_t}(\hat z_{t-1}) - F_{r_t}(\hat z_{t-1})\big)}_{\text{(A) module approximation at }\hat z_{t-1}} +\]
\[
\underbrace{\big(F_{r_t}(\hat z_{t-1}) - F_{r_t}(z_{t-1})\big)}_{\text{(B) mismatch due to different inputs}}.
\]
Applying the triangle inequality gives
\begin{align*}
\delta_t
&\le
\|\phi_{r_t}(\hat z_{t-1}) - F_{r_t}(\hat z_{t-1})\|_2
+
\|F_{r_t}(\hat z_{t-1}) - F_{r_t}(z_{t-1})\|_2.
\end{align*}
At this point, we use the fact that we are analyzing KML under \emph{learned-module Lipschitzness}.
To bound term (B) without requiring Lipschitzness of $F_{r_t}$, we express the mismatch through the learned module:
\begin{equation}
\label{eq:Ft_decomposition}
\begin{aligned}
F_{r_t}(\hat z_{t-1}) - F_{r_t}(z_{t-1})
\;=\;&
\big(F_{r_t}(\hat z_{t-1})-\phi_{r_t}(\hat z_{t-1})\big) \\
&+
\big(\phi_{r_t}(\hat z_{t-1})-\phi_{r_t}(z_{t-1})\big) \\
&+
\big(\phi_{r_t}(z_{t-1})-F_{r_t}(z_{t-1})\big).
\end{aligned}
\end{equation}

Taking norms and applying triangle inequality again yields
\begin{align*}
\|F_{r_t}(\hat z_{t-1}) - F_{r_t}(z_{t-1})\|_2
&\le
\|F_{r_t}(\hat z_{t-1})-\phi_{r_t}(\hat z_{t-1})\|_2 \\
&\quad+
\|\phi_{r_t}(\hat z_{t-1})-\phi_{r_t}(z_{t-1})\|_2 \\
&\quad+
\|\phi_{r_t}(z_{t-1})-F_{r_t}(z_{t-1})\|_2.
\end{align*}
Now, by the definition of $\varepsilon_t$ in \eqref{eq:eps_def}, both approximation terms are bounded:
\[
\|F_{r_t}(\hat z_{t-1})-\phi_{r_t}(\hat z_{t-1})\|_2 \le \varepsilon_t,
\qquad
\|\phi_{r_t}(z_{t-1})-F_{r_t}(z_{t-1})\|_2 \le \varepsilon_t,
\]
because $\|\hat z_{t-1}\|_2=\|z_{t-1}\|_2=1$ by assumption.
For the middle term, we apply Lipschitzness of the learned module \eqref{eq:lipschitz_phi}:
\[
\|\phi_{r_t}(\hat z_{t-1})-\phi_{r_t}(z_{t-1})\|_2 \le L_t\|\hat z_{t-1}-z_{t-1}\|_2 = L_t\,\delta_{t-1}.
\]
Combining these bounds gives
\[
\|F_{r_t}(\hat z_{t-1}) - F_{r_t}(z_{t-1})\|_2 \le 2\varepsilon_t + L_t\,\delta_{t-1}.
\]
Returning to the earlier inequality for $\delta_t$,
\[
\delta_t
\le
\underbrace{\|\phi_{r_t}(\hat z_{t-1}) - F_{r_t}(\hat z_{t-1})\|_2}_{\le \varepsilon_t}
+
\underbrace{\|F_{r_t}(\hat z_{t-1}) - F_{r_t}(z_{t-1})\|_2}_{\le 2\varepsilon_t + L_t\delta_{t-1}},
\]
we obtain
\[
\delta_t \le \varepsilon_t + (2\varepsilon_t + L_t\delta_{t-1}) = L_t\delta_{t-1} + 3\varepsilon_t.
\]
Finally, we absorb the constant factor into the definition of $\varepsilon_t$ (or define $\tilde\varepsilon_t:=3\varepsilon_t$),
which yields the clean recursion \eqref{eq:one_step_recursion}. \hfill $\blacksquare$

\paragraph{Remark (about constants).}
The derivation above shows that, under only the Lipschitzness of $\phi_{r_t}$, the one-step recursion holds up to a constant factor
multiplying $\varepsilon_t$; this factor depends on how the ``ideal'' operator $F_{r_t}$ is defined.
For clarity in the main theorem, we use the compact form \eqref{eq:one_step_recursion} with $\varepsilon_t$ interpreted as a uniform bound
on the \emph{effective} one-step approximation error (which may include a small constant factor).

\bigskip
\begin{theorem}[(Deterministic composition error bound over a $T$-hop traversal).]
\emph{
Assume unit-normalized states and \eqref{eq:lipschitz_phi} for the learned modules.
Then the final-state error after executing a length-$T$ program satisfies
\begin{equation}
\label{eq:main_comp_bound}
\delta_T
\le
\left(\prod_{i=1}^{T} L_i\right)\delta_0
+
\sum_{t=1}^{T}
\left(\prod_{i=t+1}^{T} L_i\right)\varepsilon_t,
\end{equation}
where by convention $\prod_{i=T+1}^{T} L_i := 1$.
In particular, if the initial state is exact ($\delta_0=0$), then
\begin{equation}
\label{eq:main_comp_bound_no_init}
\delta_T
\le
\sum_{t=1}^{T}
\left(\prod_{i=t+1}^{T} L_i\right)\varepsilon_t.
\end{equation}
}
\end{theorem}

\paragraph{Proof.}
From Lemma~1 we have the recursion
\[
\delta_t \le L_t\delta_{t-1} + \varepsilon_t.
\]
We expand this inequality iteratively.

\textbf{Step 1 (expand $\delta_1$).}
\[
\delta_1 \le L_1\delta_0 + \varepsilon_1.
\]

\textbf{Step 2 (expand $\delta_2$ using the bound on $\delta_1$).}
\begin{align*}
\delta_2
&\le L_2\delta_1 + \varepsilon_2\\
&\le L_2(L_1\delta_0 + \varepsilon_1) + \varepsilon_2\\
&= (L_2L_1)\delta_0 + (L_2)\varepsilon_1 + \varepsilon_2.
\end{align*}

\textbf{Step 3 (expand $\delta_3$).}
\begin{align*}
\delta_3
&\le L_3\delta_2 + \varepsilon_3\\
&\le L_3\big((L_2L_1)\delta_0 + (L_2)\varepsilon_1 + \varepsilon_2\big) + \varepsilon_3\\
&= (L_3L_2L_1)\delta_0 + (L_3L_2)\varepsilon_1 + (L_3)\varepsilon_2 + \varepsilon_3.
\end{align*}

\textbf{Pattern.}
After $T$ steps, each earlier error $\varepsilon_t$ is multiplied by the product of Lipschitz constants of all
\emph{subsequent} modules, because it is propagated through the remaining hops.
Thus,
\[
\delta_T
\le
\left(\prod_{i=1}^{T} L_i\right)\delta_0
+
\sum_{t=1}^{T}
\left(\prod_{i=t+1}^{T} L_i\right)\varepsilon_t,
\]
which is exactly \eqref{eq:main_comp_bound}. If $\delta_0=0$, we obtain \eqref{eq:main_comp_bound_no_init}. \hfill $\blacksquare$

\bigskip
\noindent\textbf{Corollary 1 (Simplified bounds for uniform Lipschitzness and uniform approximation error).}
\emph{
Assume $L_t \le L$ and $\varepsilon_t \le \varepsilon$ for all $t=1,\ldots,T$, and $\delta_0=0$.
Then:
\begin{itemize}
\item If $L=1$, the error grows at most linearly with hop length:
\begin{equation}
\label{eq:cor_L1}
\delta_T \le T\,\varepsilon.
\end{equation}
\item If $L\neq 1$, the error is bounded by a geometric series:
\begin{equation}
\label{eq:cor_Lneq1}
\delta_T
\le
\varepsilon\sum_{k=0}^{T-1}L^{k}
=
\varepsilon\frac{L^{T}-1}{L-1}.
\end{equation}
\item In the contractive case $L<1$, the error is uniformly bounded as $T\to\infty$:
\begin{equation}
\label{eq:cor_contract}
\delta_T \le \frac{\varepsilon}{1-L}.
\end{equation}
\end{itemize}
}

\paragraph{Proof.}
Starting from \eqref{eq:main_comp_bound_no_init} and using $L_i\le L$ and $\varepsilon_t\le\varepsilon$,
\[
\delta_T
\le
\sum_{t=1}^{T}
\left(\prod_{i=t+1}^{T} L_i\right)\varepsilon_t
\le
\varepsilon\sum_{t=1}^{T} L^{T-t}
=
\varepsilon\sum_{k=0}^{T-1}L^k.
\]
Evaluating the sum gives \eqref{eq:cor_L1} when $L=1$ and \eqref{eq:cor_Lneq1} when $L\neq 1$.
If $L<1$, then $\sum_{k=0}^{T-1}L^k \le \sum_{k=0}^{\infty}L^k = 1/(1-L)$, yielding \eqref{eq:cor_contract}. \hfill $\blacksquare$

\paragraph{Interpretation for KML.}
Theorem~2 shows that deterministic errors accumulate along a traversal in a structured way:
(i) each hop contributes its own approximation error $\varepsilon_t$,
(ii) errors injected early in the program are amplified by subsequent module Lipschitz constants,
and (iii) if the learned modules are near-non-expansive (e.g., $L_t\approx 1$ due to bounded activations such as $\tanh$
and regularized weights), then the overall error grows slowly with traversal length. This provides
a deterministic stability guarantee for KML execution over multi-hop programs.

\section{Experiments}
\label{sec:exp}

\subsection{Experimental resources}
We conduct experiments on the \dataset benchmark to evaluate the effectiveness of VLMs and Neurosymbolic methods for procedural knowledge-based question answering. 
We use NVIDIA A100 GPUs (80GB VRAM) for conducting experiments with VLMs, NVIDIA GeForce RTX 2080 Ti and A5000 GPUs (24GB VRAM) for training KML, and RTX 3090 GPUs (24GB VRAM) for training knowledge graph embedding methods.

\subsection{KML training details}

We explore three variants of KML, differing in how the embedding function $x(;\theta_x)$ is implemented. In \textbf{\kmlfclip}, $x(;\theta_x)$ is a frozen CLIP text encoder. In \textbf{\kmlclip}, we initialize $x(;\theta_x)$ with CLIP embeddings and fine-tune its parameters. In \textbf{\kmlrand}, $\theta_x$ is learned from scratch (i.e., randomly initialized).
To train KML, we use batch size of 256 triplets with AdamW~\cite{loshchilov2017decoupled} optimizer (0.01 learning rate), and 100 training epochs. Programs are generated by GPT-4o by default.
For \skshot, KML modules are fine-tuned for 100 additional epochs on video-QA pairs (0.001 learning rate). All \kmlfclip models employ 2-layer networks (128 hidden dimensions) with  CLIP~\cite{radford2021learning} (ViT-B/32) text embeddings.

\subsection{VLM Evaluation}
We evaluate VLMs under five different settings to assess their procedural knowledge reasoning ability: (1) \textbf{\svlmzero}: VLMs take a video segment (8 uniformly sampled frames), a question, and options as input; (2) \textbf{\svlmzeroplus}: VLMs take a video segment (8 uniformly sampled frames), a question, options, and the top-5 step/task categories from P.VRL as input; (3) \textbf{\skgtrain}: VLMs are tuned using the triplet instances of our \pkg as questions using LoRA~\cite{hu2021lora} and then evaluate the fine-tuned VLMs with \dataset; (4) \textbf{\skshot}: VLMs are tuned using LoRA for 100 epochs with 4 randomly sampled frames as input, demonstrating the impact of few-shot fine-tuning. All models are trained using the training set of 1,700 question-answer pairs; (5) \textbf{\skgkshot}: VLMs are tuned using triplet instances following \skgtrain, and 1,700 question-answer pairs as described in \skshot.

\subsection{Neurosymbolic Methods} 

We compare our \skml against the following NS methods.

\noindent
\textbf{Inference by Graph Propagation (\INGP)}:We implement a simple NS baseline that uses a given program and a set of grounded entities with associated probabilities or logits to answer questions. Given a directed knowledge graph with binary relations, we propagate uncertainty from grounded entities through the relations specified in the program. Using breadth-first traversal, we accumulate scores at each target entity by summing the propagated logits. This approach resembles probabilistic logic-based inference.

\noindent
\textbf{KG Embedding Methods:}
We compare against standard KG embedding models, including \transe~\cite{bordes2013translating}, \transh~\cite{wang2014knowledge}, and \rotate~\cite{sun2019rotate}. After training, we use LLM-generated programs (as in KML) to perform multi-hop reasoning. These models are selected for their support of compositional reasoning.
Embedding dimensions are tuned on a validation set, yielding optimal sizes of 256 for \transe, 100 for \transh, and 256 for \rotate. We also include variants with CLIP-initialized entity and relation embeddings to enable direct comparison with CLIP-based KML models.

\noindent
\textbf{Modern NS methods}: We compare with modern NS methods including ViperGPT~\cite{suris2023vipergpt} that uses the power of LLM and vision models, and MAC~\cite{hudson2018compositional}. We evaluate MAC on a classification-based VQA task, using a single image frame and a question (without answer choices) as input. The model predicts from 2,079 answer classes aggregated from the dataset. We use GloVe~\cite{pennington2014glove} embeddings for text and extract visual features with pretrained ResNet101~\cite{he2016deep}. We tuned the embedding size, MAC hidden size, and the number of MAC layers, selecting the best setup based on validation performance.

\subsection{Metrics} 
Since VLM-generated text may not exactly match the predefined multiple-choice options, we adopt the filtering and MCQ accuracy computation strategy from~\cite{yue2023mmmu,lin2023vila}. We report both overall accuracy and mean accuracy for each model. \textbf{Accuracy} is computed as the average score across all test samples, while \textbf{mean accuracy} (mAcc)s is the average of per-template accuracies, providing equal weight to each traversal template.

\subsection{Analysis and Discussion}
\begin{table}[h!]
\centering    
\begin{tabular}{l|l|c|c}
\toprule

Setting & \textbf{Model} & \textbf{Acc} &  \textbf{mAcc}  \\ \midrule        

\multirow{5}{*}{\svlmzero}  
&DeepSeek-VL2 (27.4B) \cite{wu2024deepseek}& 58.4 & 55.4 \\ 
&MiniCPM-V (8B) \cite{yao2024minicpm}& 62.6 & 59.7  \\ 
&mPLUG-Owl3 (7B) \cite{ye2024mplug}& \un{63.1} & \un{60.2} \\ 
&Qwen2.5-VL (7B) \cite{bai2025qwen2}& 59.6 & 57.8 \\ 
&VideoChat2-HD (7B) \cite{li2023videochat}& 61.2 & 58.4  \\ 

\midrule

\multirow{5}{*}{VLM + P.VRL } 
&DeepSeek-VL2 (27.4B) \cite{wu2024deepseek}& 64.5 & 	59.9  \\ 
&MiniCPM-V (8B) \cite{yao2024minicpm}& 67.4 & 	63.8 \\ 
&mPLUG-Owl3 (7B) \cite{ye2024mplug}& 65.5 & 	61.6  \\ 
&Qwen2.5-VL (7B) \cite{bai2025qwen2}& \un{69.4} & 	\un{65.8} \\ 
&VideoChat2-HD (7B) \cite{li2023videochat}& 65.5 & 	59.9 \\ 

\midrule

\multirow{3}{*}{\skgtrain} 
& MiniCPM-V (8B)~\cite{yao2024minicpm} & 63.5 &  	61.0  \\ 
& mPLUG-Owl3 (7B) \cite{ye2024mplug} & 64.8  &  	61.4 \\
& Qwen2.5-VL (7B) \cite{bai2025qwen2} & \un{67.3} &  \un{64.1}	 	\\
\midrule
\multirow{3}{*}{\skshot} & MiniCPM-V (8B)~\cite{yao2024minicpm} & 71.1  &  	71.4   \\ 
& mPLUG-Owl3 (7B) \cite{ye2024mplug} &  71.8 &  72.4	\\
& Qwen2.5-VL (7B) \cite{bai2025qwen2} &                 \un{73.6} & \un{73.4} \\

\midrule
\multirow{3}{*}{\skgkshot} & MiniCPM-V(8B)~\cite{yao2024minicpm} & 72.1 &  	72.1 \\ 
 & mPLUG-Owl3 (7B) \cite{ye2024mplug} &  73.1 &  \textbf{73.8}	\\
 & Qwen2.5-VL (7B) \cite{bai2025qwen2} &        
\textbf{74.2} & \textbf{73.8} \\

\bottomrule           
\end{tabular}

\caption{
Comparison of VLMs in different settings. Underlined scores denote the best-performing method within each setting, while bold scores highlight the best overall.
}
\label{tab:main.results}
\end{table}

\begin{table}[t]
\centering
\begin{tabular}{
  >{\raggedright\arraybackslash}p{2.5cm} |
  >{\raggedright\arraybackslash}p{3cm} |
  c | c 
}
\toprule
\textbf{Setting} & \textbf{Model} & \textbf{Acc} &  \textbf{mAcc}  \\         \midrule         
\multirow{2}{*}{\shortstack[l]{No Training \\ / No Program}} & ViperGPT & 41.6 &  	40.9 \\ 

&GPT-4o + P.VRL & \un{71.2} &  	\un{69.0} \\
\midrule 
Program Only &  \INGP (+P.VRL)   &  62.8   & 60.0  \\
\midrule
\skshot Only  & MAC & 11.6 & 20.0\\ \midrule
\multirow{9}{*}{\skgtrain}
& TransE & 63.6 & 51.6 \\ 
 & TransH & 73.1 & 66.3 \\ 
& RotatE & 41.6 & 29.2 \\ 
& TransE+CLIP & 56.8 & 45.4 \\ 
 & TransH+CLIP & 70.6 & 65.9 \\ 
& RotatE+CLIP & 48.8 & 35.2 \\ 
 &  \kmlfclip (Ours) &  74.6  & \un{71.6} \\ 
 &  \kmlrand (Ours) &  73.5  & 70.0 \\ 
 &  \kmlclip (Ours) &  \un{75.3}  & 71.5 \\ \midrule
 
\multirow{3}{*}{\skgkshot}  & \kmlfclip (Ours)   &  76.7 &  75.3  \\
& \kmlrand (Ours)   &  77.4 &  76.3  \\
& \kmlclip (Ours)   &  \textbf{78.1} &  \textbf{77.1}  \\


\bottomrule          
    
\end{tabular}
\caption{
Performance comparison of NS methods.}

\label{tab:main.results.ns}
\vspace{-1.5em}
\end{table}

\paragraph{Benchmarking VLMs on \dataset.}

\Cref{tab:main.results} compares the performance of various VLMs across five training and inference settings, revealing several key insights. \textbf{Providing predicted step and task} information from ProceduralVRL (VLM+P.VRL) consistently improves performance over the \textbf{zero-shot} setting. These gains highlight the importance of grounded procedural context in enhancing reasoning, even for strong models like MiniCPM-V and Qwen2.5-VL.
\textbf{Training on KG triplets} yields some improvement over zero-shot baselines, though the gains are modest and less consistent, suggesting that aligning symbolic representations with multimodal inputs remains non-trivial. 
\textbf{QA-based fine-tuning} leads to larger improvements, particularly for MiniCPM-V and Qwen2.5-VL. 
These gains indicate that VLMs are capable of adapting to task-specific supervision, even when provided with a relatively small number of QA pairs (1,700 samples).
The best performance is achieved when \textbf{combining both KG and QA training}, with Qwen2.5-VL reaching 74.2/73.8, indicating that integrating structured procedural knowledge with task-specific examples provides complementary benefits for enhancing procedural understanding in VLMs.

\paragraph{Benchmarking Neurosymbolic Methods.}
As shown in Table~\ref{tab:main.results.ns}, all KML variants outperform all baselines, demonstrating the benefit of executing LLM-generated KG traversal programs with relation-specific KM for the proposed procedural reasoning question answering task. 
KML achieves the highest performance under \skgkshot using \kmlclip variant.
The performance of \kmlrand is not far from \kmlfclip and interestingly, when tuned with \skgkshot the \kmlrand performs better than \kmlfclip under the same setting.
\textbf{ViperGPT} performs poorly on our benchmark due to the absence of built-in reasoning operators and reliance on predefined program templates that are not suited for the structured reasoning required in procedural knowledge tasks.
\textbf{\INGP} suffers from combinatorial explosion and poor grounding, often generating irrelevant results from uncertain starting nodes. In contrast, KML handles such uncertainty better by grounding traversal via learned knowledge modules.
\textbf{MAC}, originally designed for visual reasoning tasks (e.g., answering questions like “What is to the left of the green box?”), fails in our setting due to its lack of access to external procedural knowledge. 

Among \textbf{KG embedding methods}, TransE and TransH perform better due to their additive or projective composition, aligning well with our program-based reasoning. RotatE underperforms, likely because its complex-valued, rotation-based composition is less effective for multi-hop reasoning. CLIP-based initialization yields mixed results—improving RotatE but degrading TransE and TransH—indicating varying alignment with visual semantics across models.
We also evaluate \textbf{GPT-4o} using top-5 predicted step/task category names from P.VRL. It achieves 71.2\% accuracy and 69.0\% mean accuracy.

\subsection{KML's Generalizability.} 
To assess KML's generalizability, we added ten binary relations from the \textbf{STAR benchmark}~\cite{wu2024star}. 
We use CLIP~\cite{radford2021learning} model for entity recognition such as action, relation, object. Specifically, we fine-tune CLIP vision model on the provided annotations of Action Genome dataset and the STAR Benchmark~\cite{wu2024star} for visual entity recognition using contrastive learning.  
We design 10 knowledge modules from the hyper-graphs of the training set of the dataset.
\kmlfclip achieved 74.9\% on Interaction and 76.7\% on Feasibility questions, outperforming prior bests of 71.8\%~\cite{jaiswal2025learning} and 62.4\%~\cite{yu2023self}. For Sequence and Prediction questions, it scored 57.3\% and 49.8\%, respectively.

\textbf{Training with other KGs.}
We also trained \kmlfclip using a \textbf{GPT-4o-generated KG} from 7,687 WikiHow tasks (57,027 steps, 8.7M triplets) capturing tools, actions, objects, and purposes. Evaluated on our \dshort~ dataset, it achieved 73.9\% mean accuracy, rising to 74.9\% with \skgkshot, suggesting that while generic KGs help, domain-specific modules remain advantageous.

\subsection{KML’s Robustness.} \Cref{fig:topk} (left) shows KML’s performance using top-$k$ step/task predictions ($k=1$ to $5$). Using all top-5 predictions yields the best QA performance, while top-1 predictions still achieve strong results, demonstrating the model’s robustness to imperfect inputs.
\Cref{fig:topk} (right) reports a moderate correlation (0.24) between step prediction accuracy and QA accuracy, suggesting KML does not heavily depend on input quality, benefiting from embedding-based reasoning and generalizable knowledge modules.
\begin{figure}[h!]
    \centering
    \includegraphics[width=0.45\linewidth]{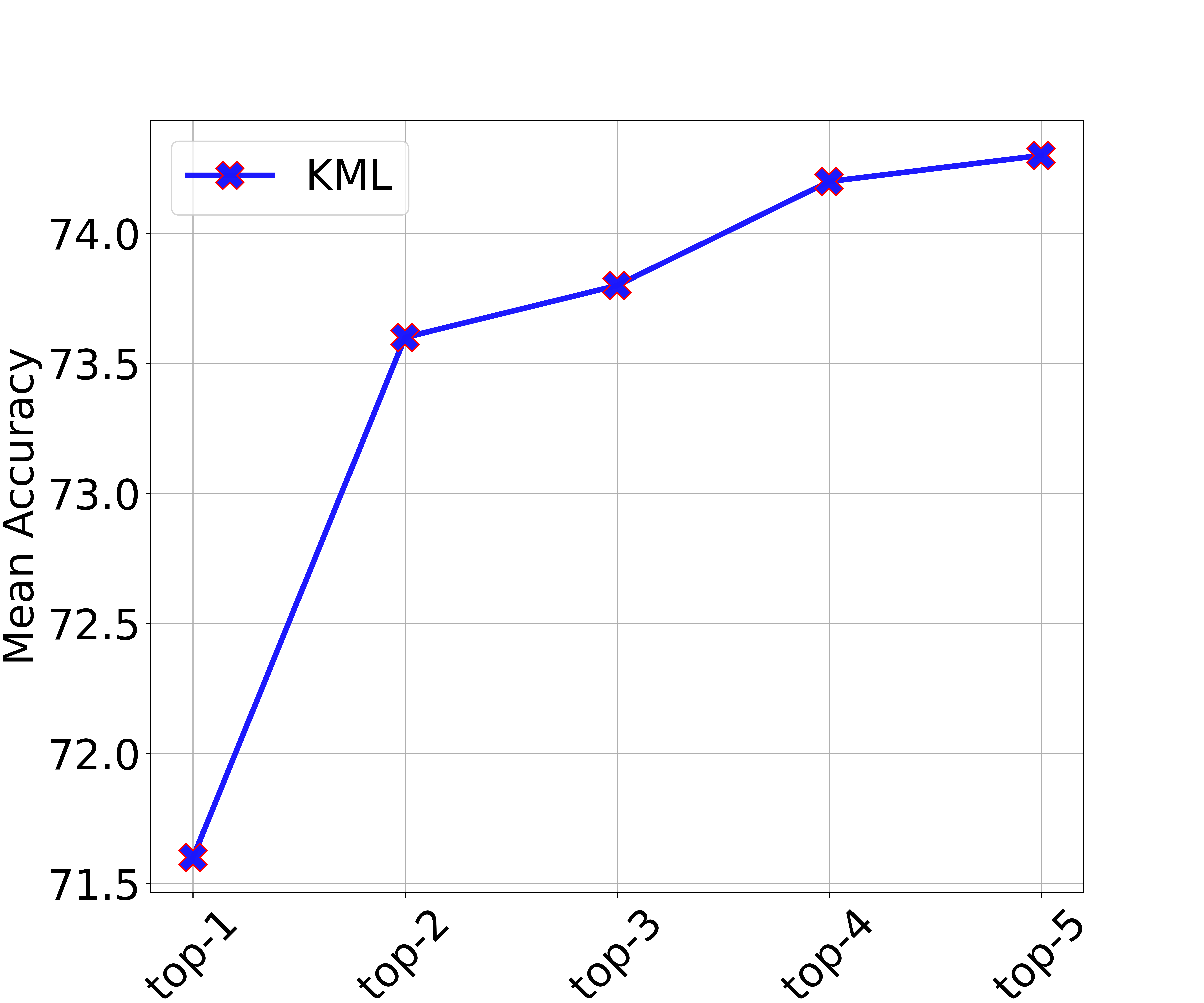}
    \includegraphics[width=0.5\linewidth]{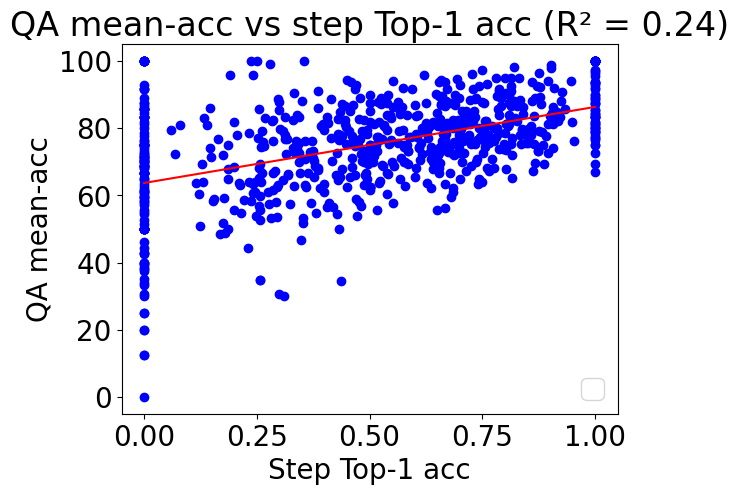}
    \caption{    
    (Left) QA performance of \kmlfclip using top-1 to top-5 grounded entities from P.VRL. (Right) Correlation between step prediction accuracy and QA accuracy.
    }
    \label{fig:topk}    
\end{figure}

\subsection{KML Ablation.} 
We further study the impact of various components of our procedural knowlde question answering model shown in~\Cref{fig:model}.
First, we ablated task/step grounding using VLM (P.VRL). Without grounding, performance dropped to 44.0\%/42.4\%.
Then we replace the LLM-generated program with programs generated exhaustively constraining the program length to maximum of 3-hops. Given a question, we select the program using the inference mechanism of KML.
Replacing LLM-generated programs with exhaustive programs yielded 
57.9\%/47.7\%. 
These results show that grounding, program generation, and knowledge modules each contribute substantially to the performance.
We train KML modules from scratch using only \skshot, with \kmlfclip achieving 59.3\% mean accuracy—highlighting the value of training on \pkg data.
KML allows exploration of multiple programs per question, though gains over single-program use are marginal. It also supports expert program editing for improved reliability. 
\begin{table*}[t]
\centering
\resizebox{\linewidth}{!}{%
\begin{tabular}{lccccccc}
\hline
\textbf{Method} & \textbf{Grounding (P.VRL)} & \textbf{LLM-Gen. Program} & \textbf{KM} &  \textbf{Use KG} &\textbf{Accuracy} & \textbf{Mean Accuracy} \\
\hline
Without grounding (use video feature)  &  & \checkmark & \checkmark & \checkmark & 44.0 & 42.4 \\
With Random Program (minimum energy)  & \checkmark  &  & \checkmark & \checkmark &57.9 & 47.7 \\
With Random Init. KML & \checkmark  & \checkmark & \checkmark &  &52.8 & 59.3\\
Without Knowledge Module (IGP)  & \checkmark & \checkmark & \checkmark & \checkmark &62.8 & 60.0 \\
KML-CLIP (Our full model)  & \checkmark & \checkmark & \checkmark & \checkmark &75.3 & 71.5 \\
\hline
\end{tabular}}
\caption{Ablation results showing the impact of KG training, grounding, LLM program generation, and knowledge modules.}
\label{tab:ablation}
\end{table*}

\subsection{Impact of the type of LLM for program generation} 
We evaluate different LLMs for program generation using \kmlfclip in~\Cref{tbl.llm}.
Most modern LLMs are good at generating the program given the question and extracting the grounded entity. We choose GPT-4o as it performs the best.
\begin{table}[t]
    \centering
    \begin{tabular}{c|c|c} \hline
        LLM & acc & m.acc  \\ \hline
        GPT-4o  & 74.6 & 71.6 \\
        LLaMA-3-8B  & 74.1 & 68.4 \\
        DeepSeek-V2.5  & 73.4 & 66.5 \\
        Mistral-7B  & 73.2 & 66.2 \\
        Qwen-2.5  & 72.6 & 63.7 \\ \hline        
    \end{tabular}
    \caption{Impact of LLMs for program generation.}
    \label{tbl.llm}
\end{table}

\subsection{Hyper-Parameter Learning}
\label{sec.hyper.param.kml}
We evaluate the impact of the hidden size of our knowledge modules. As shown in \Cref{fig:validation.hdim}, a small hidden size of 128 is sufficient to learn all modules in the KG, using a total number of 2,107,393 parameters. 
\begin{figure}[h!]
    \centering
    \includegraphics[width=0.49\linewidth]{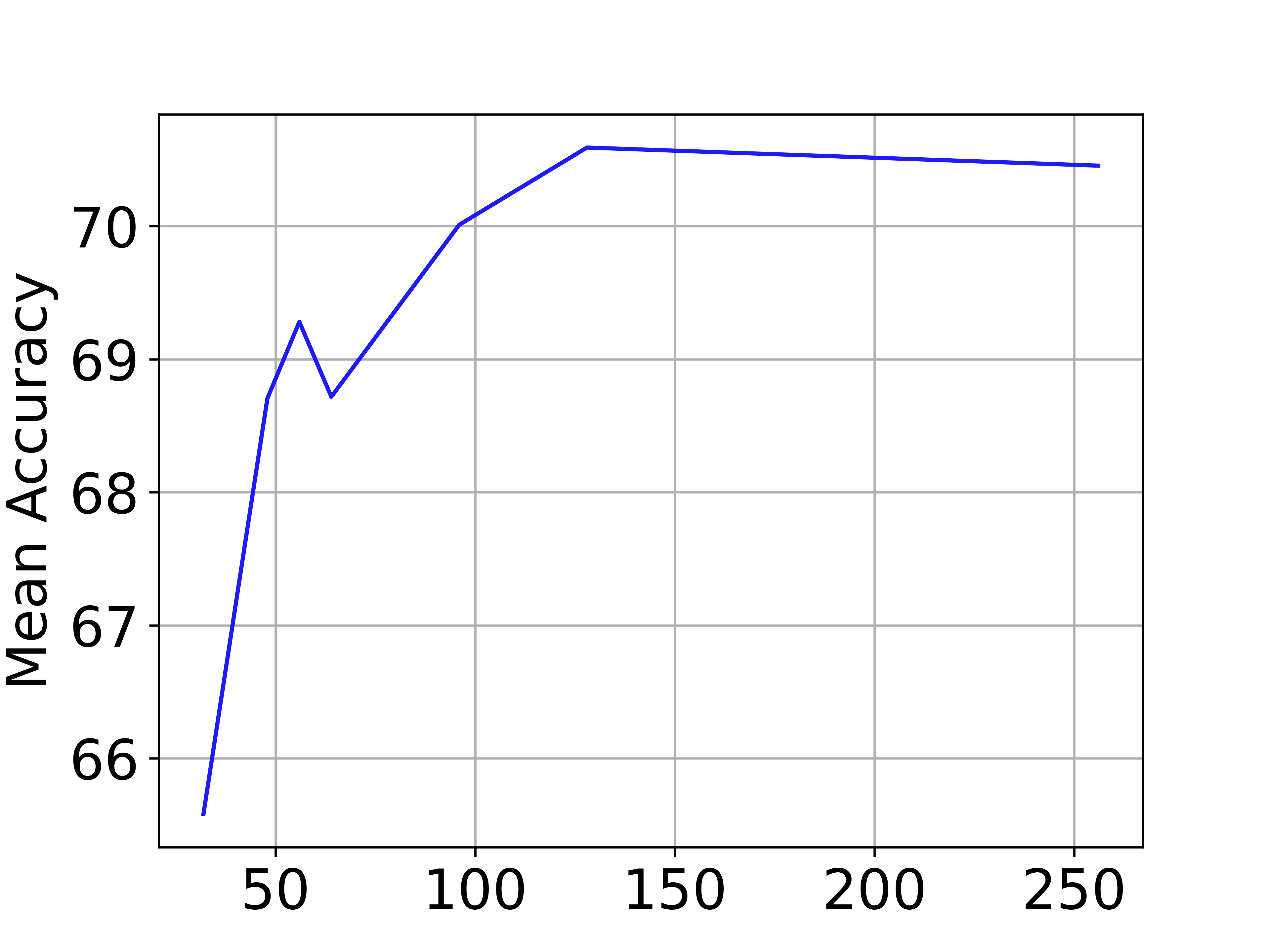}
    \includegraphics[width=0.49\linewidth]{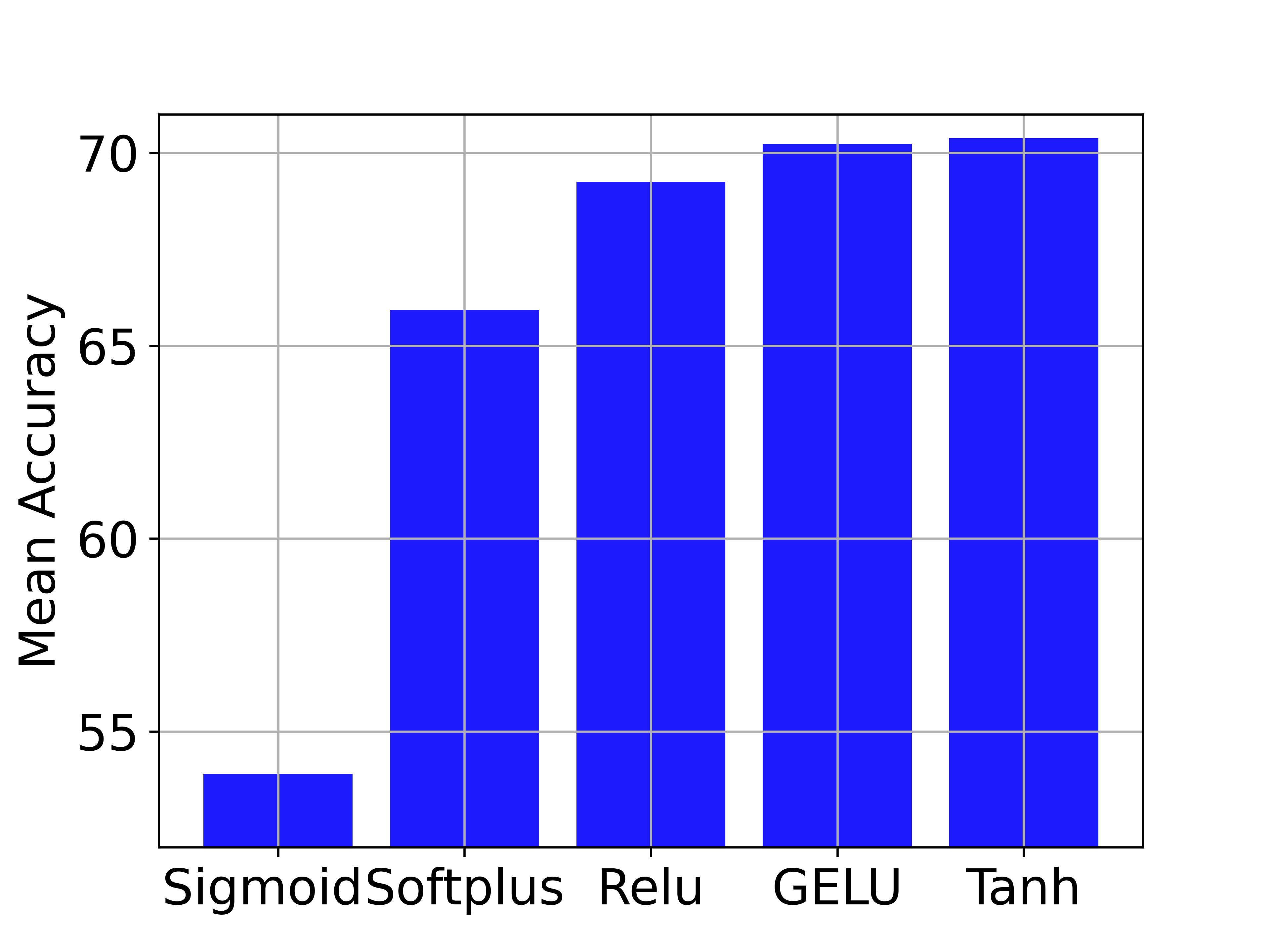}
        \caption{Evaluating the impact of hidden dimension size (left)  and the activation function (right) of the Knowledge modules on the validation set.}
    \label{fig:validation.hdim}
\end{figure}
We also experimented with MSE loss to train the modules, however, the performance was poorer and the combination of contrastive loss and MSE loss did not perform as well as classical contrastive loss. The Tanh activation performs the best among all compared activation functions perhaps due to bounded Lipschitz constant of 1.
Sigmoid is more contractive with  Lipschitz constant of $\frac{1}{4}$ which can help stability but contributes to vanishing gradients.
The grounding entity $E_g$ of \Cref{eq.eg} is 100\% accurately found by all LLM models.

\subsection{KML's logical operator performance}

We now extend KML to support logical operators such as $AND$, $OR$, and $NOT$ over knowledge graph relations.
Consider a binary relation $r_k(\cdot,\cdot)$ and two head entities $e_i$ and $e_j$.
Our goal is to infer the set of entities $\mathcal{E}_{ijk}$ that satisfy the logical conjunction
$r_k(e_i,\cdot) \land r_k(e_j,\cdot)$, i.e., entities that are simultaneously related to both $e_i$
and $e_j$ under relation $r_k$.

Formally, this set is defined as
\[
\mathcal{E}_{ijk}
=
\bigl\{
e \in \mathcal{E}
\;\big|\;
r_k(e_i, e) \;\land\; r_k(e_j, e)
\bigr\}.
\]

We compute entity-level scores and perform logical operations directly in the score space.
Let $\mathbf{E} \in \mathbb{R}^{|\mathcal{E}| \times d}$ denote the matrix of all entity embeddings,
and let $\mathbf{e}_i$ and $\mathbf{e}_j$ denote the output embeddings produced by the knowledge module
$\phi_{R_k}$ for entities $e_i$ and $e_j$, respectively.
Entity scores under relation $r_k$ are obtained via sigmoid activated similarities.
Next we implement logical operators as follows:

\noindent
\textbf{Logical AND.}
The score of entities satisfying the logical $AND$ is defined as
\begin{equation}
\text{score}_{AND}(\mathcal{E}_{ijk})
=
\sigma\!\left(\mathbf{E}\mathbf{e}_i^{\top}\right)
\;\circ\;
\sigma\!\left(\mathbf{E}\mathbf{e}_j^{\top}\right),
\end{equation}
where $\circ$ denotes element-wise multiplication and $\sigma(\cdot)$ is the sigmoid function.
This operation corresponds to a soft intersection, assigning high scores only to entities that
score highly for both relations.

\noindent
\textbf{Logical OR.}
The logical $OR$ operator is implemented as
\begin{equation}
\text{score}_{OR}(\mathcal{E}_{ijk})
=
\sigma\!\left(\mathbf{E}\mathbf{e}_i^{\top}\right)
+
\sigma\!\left(\mathbf{E}\mathbf{e}_j^{\top}\right),
\end{equation}
which corresponds to a soft union, favoring entities that satisfy at least one of the relations.

\noindent
\textbf{Logical NOT.}
Finally, the logical $NOT$ operator is implemented as
\begin{equation}
\text{score}_{NOT}\bigl(r_k(e_i,\cdot)\bigr)
=
1 -
\sigma\!\left(\mathbf{E}\mathbf{e}_i^{\top}\right),
\end{equation}
which assigns high scores to entities that are unlikely to be related to $e_i$ under relation $r_k$.

To validate this, we conducted an experiment on learned knowledge modules of each relation type.
We report precision at 10 performance for each module in~\Cref{tab:km_logical_ops}.
As can be seen from the results, the performance of NOT and AND operator for most of the relations is good while the performance of OR operator is not that significant.  The reason may be the logical OR operator will generate a large number of candidates for each relation and the simple scoring mechanism may not be that effective in this context. Perhaps, to implement logical $OR$ we may need to learn explicit vector function forms.
\begin{table}[t]
\centering
\begin{tabular}{lccc}
\toprule
\textbf{Knowledge Module} & \textbf{NOT} & \textbf{AND} & \textbf{OR} \\
\midrule
HAS\_GROUNDED\_TOOL      & 100.0 & 94.4 & 24.6 \\
HAS\_PURPOSE             & 100.0 & 91.0 & 27.6 \\
HAS\_NEXT\_STEP          & 100.0 & 60.0 & 82.5 \\
HAS\_TASK                & 100.0 & --   & 70.9 \\
HAS\_STEP                & 100.0 & 97.4 & 37.9 \\
HAS\_ACTION              & 100.0 & 100.0 & 45.1 \\
HAS\_OBJECT              & 100.0& 99.8 & 22.3 \\
HAS\_SIMILAR\_PURPOSE    & 100.0 & 65.9 & 13.9 \\
\bottomrule
\end{tabular}
\caption{Precision at 10 accuracy (\%) of individual Knowledge Modules under different logical operators.}
\label{tab:km_logical_ops}
\end{table}

\subsection{KML's Interpretability.} 
One key advantage of KML is its interpretability, as illustrated in the qualitative examples in~\Cref{fig:q1} and~\Cref{fig:q2}. We observe step-by-step reasoning and intermediate interpretations from the learned embeddings, offering insight into the model’s decision process. 
The output entities represented by the output vectors of each KM seems reasonable and accurate for the given task.

\begin{table}[t]
\centering
\resizebox{.9\linewidth}{!}{
\scriptsize
\rowcolors{2}{gray!10}{white} 
\begin{tabular}{|>{\centering\arraybackslash}p{3cm}|
                >{\centering\arraybackslash}p{3cm}|
                >{\centering\arraybackslash}p{3cm}|}

\hline
\begin{minipage}{3cm}
\includegraphics[width=3cm, trim=10 10 10 10, clip]{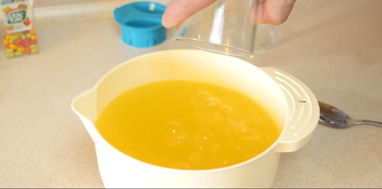}
\end{minipage}
&
Q. What is the other task that use the tool in this video for the same purpose?
& 
Task: Make Orange Juice
Step: pour the orange juice into the cup  
\\
\hline
\rowcolor{blue!25}
\textbf{HAS\_TOOL} & \textbf{TOOL\_TO\_STEP} & \textbf{STEP\_TO\_TASK} \\
\hline
\rowcolor{green!25}
Out=\textbf{Tool} & Out=\textbf{Step} & Out=\textbf{Task} \\
\hline
cup (0.361) & pour into the ingredients (0.337)
 & MakeCookie (0.221) \\
mug (0.273)
  & pour in after mix it (0.314)
 & MakeCocktail (0.208)
 \\
measuring cup (0.253)
& add some ingredients to the tea (0.308)
& MakeHomemadeIceCream (0.191) \\
yogurt (0.249)
&add some ingredients in the coffee (0.307)
& MakeChocolate (0.188) \\
bottle (0.223)
&pour the ingredients into the bowl (0.293)
& MakeCoffee (0.185) \\
\hline
\end{tabular}
}
\caption{Three-hop Reasoning using \kmlfclip: The step-by-step reasoning outputs of KML with estimated probability value over the domain of relation using embeddings.}
\label{fig:q1}
\end{table}

\begin{table}[t]
\resizebox{\linewidth}{!}{
\centering
\scriptsize
\rowcolors{2}{gray!10}{white} 
\begin{tabular}{|>{\centering\arraybackslash}p{3cm}|
                >{\centering\arraybackslash}p{3cm}|
                >{\centering\arraybackslash}p{3cm}|
                >{\centering\arraybackslash}p{3cm}|}

\hline
\multicolumn{2}{|c|}{
\begin{minipage}{3cm}
\includegraphics[width=3cm, trim=10 10 10 10, clip]{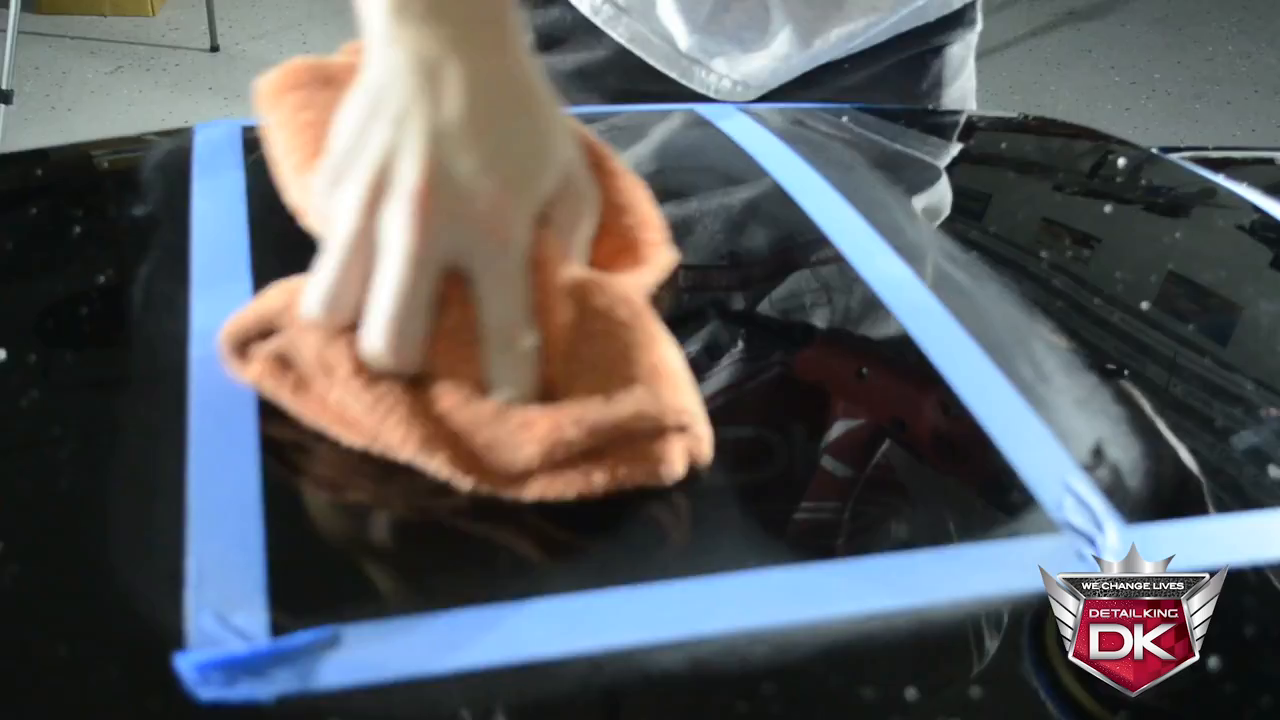}
\end{minipage}
} 
&
Q. What is an alternative tool can be used for this step?
& 
Task: Polish Car
Step: clean the scratch
\\
\hline
\rowcolor{blue!25}
\textbf{HAS\_TOOL} & \textbf{HAS\_PURPOSE} & \textbf{SIMILAR\_PURPOSE} &  \textbf{PURPOSE\_TO\_TOOL} \\
\hline
\rowcolor{green!25}
Out=\textbf{Tool} & Out=\textbf{Purpose} & Out=\textbf{Purpose} & Out=\textbf{Tool} \\
\hline
microfiber towel (0.237)  &  wiping up dust (0.249)  &  cleaning things (0.233)  &  cloth (0.263)  \\
microfiber cloth (0.229)  &  dusting (0.241)  &  cleaning (0.219)  &  towel (0.253)  \\
polishing pad (0.217)  &  cleaning (0.238)  &  wiping up wet spill (0.205)  &  soap (0.219)  \\
cloth (0.168)  &  wiping up wet spill (0.228)  &  clean dirty things (0.204)  &  curtains (0.178)  \\
towel (0.125)  &  wiping (0.215)  &  cleaning up (0.200)  &  cushion (0.178) \\ \hline
\end{tabular}
}
\caption{Four-hop Reasoning using \kmlfclip: The step-by-step reasoning outputs of KML with estimated probability value over the domain of relation using embeddings.}
\label{fig:q2}
\end{table}

\begin{table}[t!]
\scriptsize
\centering
\rowcolors{2}{gray!10}{white}
\begin{tabular}{|>{\centering\arraybackslash}m{5cm}|             
>{\centering\arraybackslash}m{3cm}|}
\hline
\raisebox{-.6\height}{\includegraphics[height=2.5cm]{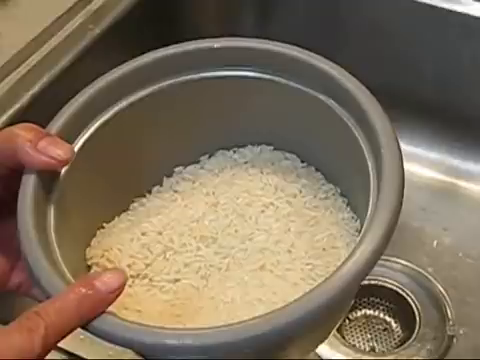}}  
&
\parbox[t]{\linewidth}{\raggedright \textbf{Question:} What tool could be used in the next step? \newline 
\textbf{Task:} Use Rice-Cooker To Cook Rice\newline 
\textbf{Step:} Take out some rice.}
\\ \hline
\rowcolor{blue!25}
\textbf{NEXT\_STEP} & \textbf{HAS\_TOOL} \\
\hline
\rowcolor{green!25}
Output=\textbf{Next Step} & Output=\textbf{Tool} \\
\hline
cook the rice by rice cooker (0.331)  &  rice cooker (0.233) \\
  
put the washed rice into the rice cooker (0.311)  &  water (0.201) \\
  
soak and wash the rice (0.271)  &  bowl (0.197) \\
  
take out some rice (0.154)  &  water container (0.174) \\
  
pour the noodles into the water and stir (0.138)  &  water basin (0.167) \\
\hline
\end{tabular}
\caption{
Step-by-step reasoning of our KML method for a given question. The program \texttt{[NEXT\_STEP, HAS\_TOOL]} is executed by corresponding Knowledge Modules, producing intermediate embeddings with interpretable semantic meaning. This enables explanation of the final output. Experts can modify the program to explore alternative reasoning paths. The program is generated by an LLM based on the question and KG schema.
}
\label{fig:illustrate1}
\end{table}

\section{Conclusion}
We presented an extended formulation of Knowledge Module Learning (KML) as a principled neuro-symbolic framework for procedural knowledge reasoning. KML decomposes reasoning into relation-specific neural modules and composes them through explicit programmatic traversals, enabling interpretable and modular reasoning over structured procedural knowledge graphs. In contrast to black-box vision–language models, KML exposes intermediate reasoning states while retaining the generalization capability of neural representations.
This journal extension introduced a formal treatment of logical operators within KML. We showed how logical reasoning over relations can be implemented as differentiable operations in score space, providing a soft relaxation of Boolean set operations compatible with gradient-based learning. We further provided a deterministic error bound for multi-hop traversal execution, offering theoretical insight into error propagation and stability under composition of learned knowledge modules.
Experimental results demonstrate that KML consistently outperforms strong baselines on procedural reasoning tasks while producing interpretable execution traces aligned with relational semantics. These results, together with the theoretical analysis, establish KML as a robust foundation for structured and reliable procedural reasoning.

KML naturally supports extensions to richer logical forms, such as nested expressions, quantification, and temporal reasoning. An important future direction is to integrate KML with embodied reasoning, where procedural knowledge must be grounded in perception and action. Extending KML to embodied agents would enable reasoning over affordances, state transitions, and action outcomes, bridging symbolic procedural knowledge and real-world interaction.

\noindent
\textbf{Acknowledgments}. This research/project is supported by the National Research Foundation, Singapore, under its NRF Fellowship (Award\# NRF-NRFF14-2022-0001). 

{\small
\bibliographystyle{IEEEtran}
\bibliography{main}

@String(CVPR= {IEEE Conf. Comput. Vis. Pattern Recog.})

@String(AAAI = {AAAI})

@String(CVPR  = {CVPR})

@inproceedings{chang2020procedure,
  title={Procedure planning in instructional videos},
  author={Chang, Chien-Yi and Huang, De-An and Xu, Danfei and Adeli, Ehsan and Fei-Fei, Li and Niebles, Juan Carlos},
  booktitle={European Conference on Computer Vision},
  pages={334--350},
  year={2020},
  organization={Springer}
}

@article{ashutosh2024video,
  title={Video-mined task graphs for keystep recognition in instructional videos},
  author={Ashutosh, Kumar and Ramakrishnan, Santhosh Kumar and Afouras, Triantafyllos and Grauman, Kristen},
  journal={Advances in Neural Information Processing Systems},
  volume={36},
  year={2024}
}

@inproceedings{tang2019coin,
  title={Coin: A large-scale dataset for comprehensive instructional video analysis},
  author={Tang, Yansong and Ding, Dajun and Rao, Yongming and Zheng, Yu and Zhang, Danyang and Zhao, Lili and Lu, Jiwen and Zhou, Jie},
  booktitle={Proceedings of the IEEE/CVF Conference on Computer Vision and Pattern Recognition},
  pages={1207--1216},
  year={2019}
}

@inproceedings{abdelaziz2021semantic,
  title={A semantic parsing and reasoning-based approach to knowledge base question answering},
  author={Abdelaziz, Ibrahim and Ravishankar, Srinivas and Kapanipathi, Pavan and Roukos, Salim and Gray, Alexander},
  booktitle={Proceedings of the AAAI conference on artificial intelligence},
  volume={35},
  number={18},
  pages={15985--15987},
  year={2021}
}

@article{zhong2024synthet2c,
  title={SyntheT2C: Generating Synthetic Data for Fine-Tuning Large Language Models on the Text2Cypher Task},
  author={Zhong, Ziije and Zhong, Linqing and Sun, Zhaoze and Jin, Qingyun and Qin, Zengchang and Zhang, Xiaofan},
  journal={arXiv preprint arXiv:2406.10710},
  year={2024}
}

@inproceedings{yang2023llm,
  title={Llm-based sparql generation with selected schema from large scale knowledge base},
  author={Yang, Shuangtao and Teng, Mao and Dong, Xiaozheng and Bo, Fu},
  booktitle={China Conference on Knowledge Graph and Semantic Computing},
  pages={304--316},
  year={2023},
  organization={Springer}
}

@inproceedings{zhou2023procedure,
  title={Procedure-aware pretraining for instructional video understanding},
  author={Zhou, Honglu and Mart{\'\i}n-Mart{\'\i}n, Roberto and Kapadia, Mubbasir and Savarese, Silvio and Niebles, Juan Carlos},
  booktitle={Proceedings of the IEEE/CVF Conference on Computer Vision and Pattern Recognition},
  pages={10727--10738},
  year={2023}
}

@inproceedings{nagasinghe2024not,
  title={Why Not Use Your Textbook? Knowledge-Enhanced Procedure Planning of Instructional Videos},
  author={Nagasinghe, Kumaranage Ravindu Yasas and Zhou, Honglu and Gunawardhana, Malitha and Min, Martin Renqiang and Harari, Daniel and Khan, Muhammad Haris},
  booktitle={Proceedings of the IEEE/CVF Conference on Computer Vision and Pattern Recognition},
  pages={18816--18826},
  year={2024}
}

@inproceedings{miech2020end,
  title={End-to-end learning of visual representations from uncurated instructional videos},
  author={Miech, Antoine and Alayrac, Jean-Baptiste and Smaira, Lucas and Laptev, Ivan and Sivic, Josef and Zisserman, Andrew},
  booktitle={Proceedings of the IEEE/CVF conference on computer vision and pattern recognition},
  pages={9879--9889},
  year={2020}
}

@article{sun2022plate,
  title={Plate: Visually-grounded planning with transformers in procedural tasks},
  author={Sun, Jiankai and Huang, De-An and Lu, Bo and Liu, Yun-Hui and Zhou, Bolei and Garg, Animesh},
  journal={IEEE Robotics and Automation Letters},
  volume={7},
  number={2},
  pages={4924--4930},
  year={2022},
  publisher={IEEE}
}

@inproceedings{zhong2023learning,
  title={Learning procedure-aware video representation from instructional videos and their narrations},
  author={Zhong, Yiwu and Yu, Licheng and Bai, Yang and Li, Shangwen and Yan, Xueting and Li, Yin},
  booktitle={Proceedings of the IEEE/CVF Conference on Computer Vision and Pattern Recognition},
  pages={14825--14835},
  year={2023}
}

@inproceedings{marino2019ok,
  title={Ok-vqa: A visual question answering benchmark requiring external knowledge},
  author={Marino, Kenneth and Rastegari, Mohammad and Farhadi, Ali and Mottaghi, Roozbeh},
  booktitle={Proceedings of the IEEE/cvf conference on computer vision and pattern recognition},
  pages={3195--3204},
  year={2019}
}

@inproceedings{chang2022webqa,
  title={Webqa: Multihop and multimodal qa},
  author={Chang, Yingshan and Narang, Mridu and Suzuki, Hisami and Cao, Guihong and Gao, Jianfeng and Bisk, Yonatan},
  booktitle={Proceedings of the IEEE/CVF conference on computer vision and pattern recognition},
  pages={16495--16504},
  year={2022}
}

@inproceedings{park2020visualcomet,
  title={Visualcomet: Reasoning about the dynamic context of a still image},
  author={Park, Jae Sung and Bhagavatula, Chandra and Mottaghi, Roozbeh and Farhadi, Ali and Choi, Yejin},
  booktitle={Computer Vision--ECCV 2020: 16th European Conference, Glasgow, UK, August 23--28, 2020, Proceedings, Part V 16},
  pages={508--524},
  year={2020},
  organization={Springer}
}

@inproceedings{schwenk2022okvqa,
  title={A-okvqa: A benchmark for visual question answering using world knowledge},
  author={Schwenk, Dustin and Khandelwal, Apoorv and Clark, Christopher and Marino, Kenneth and Mottaghi, Roozbeh},
  booktitle={European conference on computer vision},
  pages={146--162},
  year={2022},
  organization={Springer}
}

@inproceedings{wang2024sok,
  title={SOK-Bench: A Situated Video Reasoning Benchmark with Aligned Open-World Knowledge},
  author={Wang, Andong and Wu, Bo and Chen, Sunli and Chen, Zhenfang and Guan, Haotian and Lee, Wei-Ning and Li, Li Erran and Gan, Chuang},
  booktitle={Proceedings of the IEEE/CVF Conference on Computer Vision and Pattern Recognition},
  pages={13384--13394},
  year={2024}
}

@inproceedings{speer2017conceptnet,
  title={Conceptnet 5.5: An open multilingual graph of general knowledge},
  author={Speer, Robyn and Chin, Joshua and Havasi, Catherine},
  booktitle={Proceedings of the AAAI conference on artificial intelligence},
  volume={31},
  number={1},
  year={2017}
}

@article{hurst2024gpt,
  title={GPT-4o System Card},
  author={Hurst, Aaron and Lerer, Adam and Goucher, Adam P and Perelman, Adam and Ramesh, Aditya and Clark, Aidan and Ostrow, AJ and Welihinda, Akila and Hayes, Alan and Radford, Alec and others},
  journal={arXiv preprint arXiv:2410.21276},
  year={2024}
}

@inproceedings{reimers-2019-sentence-bert,
  title = "Sentence-BERT: Sentence Embeddings using Siamese BERT-Networks",
  author = "Reimers, Nils and Gurevych, Iryna",
  booktitle = "Proceedings of the 2019 Conference on Empirical Methods in Natural Language Processing",
  month = "11",
  year = "2019",
  publisher = "Association for Computational Linguistics",
  url = "https://arxiv.org/abs/1908.10084",
}

@inproceedings{yue2023mmmu,
  title={MMMU: A Massive Multi-discipline Multimodal Understanding and Reasoning Benchmark for Expert AGI},
  author={Xiang Yue and Yuansheng Ni and Kai Zhang and Tianyu Zheng and Ruoqi Liu and Ge Zhang and Samuel Stevens and Dongfu Jiang and Weiming Ren and Yuxuan Sun and Cong Wei and Botao Yu and Ruibin Yuan and Renliang Sun and Ming Yin and Boyuan Zheng and Zhenzhu Yang and Yibo Liu and Wenhao Huang and Huan Sun and Yu Su and Wenhu Chen},
  booktitle={Proceedings of CVPR},
  year={2024},
}

@misc{lin2023vila,
      title={VILA: On Pre-training for Visual Language Models},
      author={Ji Lin and Hongxu Yin and Wei Ping and Yao Lu and Pavlo Molchanov and Andrew Tao and Huizi Mao and Jan Kautz and Mohammad Shoeybi and Song Han},
      year={2023},
      eprint={2312.07533},
      archivePrefix={arXiv},
      primaryClass={cs.CV}
}

@article{li2023videochat,
  title={Videochat: Chat-centric video understanding},
  author={Li, KunChang and He, Yinan and Wang, Yi and Li, Yizhuo and Wang, Wenhai and Luo, Ping and Wang, Yali and Wang, Limin and Qiao, Yu},
  journal={arXiv preprint arXiv:2305.06355},
  year={2023}
}

@article{yao2024minicpm,
  title={MiniCPM-V: A GPT-4V Level MLLM on Your Phone},
  author={Yao, Yuan and Yu, Tianyu and Zhang, Ao and Wang, Chongyi and Cui, Junbo and Zhu, Hongji and Cai, Tianchi and Li, Haoyu and Zhao, Weilin and He, Zhihui and others},
  journal={arXiv preprint arXiv:2408.01800},
  year={2024}
}

@inproceedings{li2024mvbench,
  title={Mvbench: A comprehensive multi-modal video understanding benchmark},
  author={Li, Kunchang and Wang, Yali and He, Yinan and Li, Yizhuo and Wang, Yi and Liu, Yi and Wang, Zun and Xu, Jilan and Chen, Guo and Luo, Ping and others},
  booktitle={Proceedings of the IEEE/CVF Conference on Computer Vision and Pattern Recognition},
  pages={22195--22206},
  year={2024}
}

@inproceedings{johnson2017inferring,
  title={Inferring and executing programs for visual reasoning},
  author={Johnson, Justin and Hariharan, Bharath and Van Der Maaten, Laurens and Hoffman, Judy and Fei-Fei, Li and Lawrence Zitnick, C and Girshick, Ross},
  booktitle={Proceedings of the IEEE international conference on computer vision},
  pages={2989--2998},
  year={2017}
}

@article{hu2021lora,
  title={Lora: Low-rank adaptation of large language models},
  author={Hu, Edward J and Shen, Yelong and Wallis, Phillip and Allen-Zhu, Zeyuan and Li, Yuanzhi and Wang, Shean and Wang, Lu and Chen, Weizhu},
  journal={arXiv preprint arXiv:2106.09685},
  year={2021}
}

@article{sun2019rotate,
  title={Rotate: Knowledge graph embedding by relational rotation in complex space},
  author={Sun, Zhiqing and Deng, Zhi-Hong and Nie, Jian-Yun and Tang, Jian},
  journal={arXiv preprint arXiv:1902.10197},
  year={2019}
}

@inproceedings{francis2018cypher,
  title={Cypher: An evolving query language for property graphs},
  author={Francis, Nadime and Green, Alastair and Guagliardo, Paolo and Libkin, Leonid and Lindaaker, Tobias and Marsault, Victor and Plantikow, Stefan and Rydberg, Mats and Selmer, Petra and Taylor, Andr{\'e}s},
  booktitle={Proceedings of the 2018 international conference on management of data},
  pages={1433--1445},
  year={2018}
}

@article{thoe2022developing,
  title={Developing Conceptual and Procedural Knowledge/Skills of Lifelong Learners from Basic to Advance Learning: Exemplars, Challenges and Future Direction},
  author={Thoe, Ng Khar and Jamaludin, Junainah and Pang, Yee Jiea and Choong, Careemah and Lay, Yoon Fah and Ong, Eng Tek and Durairaj, Kamalambal and Talib, Corrienna Abdul and Chin, Chee Keong},
  journal={Dinamika Jurnal Ilmiah Pendidikan Dasar},
  volume={14},
  number={1},
  pages={22--35},
  year={2022}
}

@inproceedings{radford2021learning,
  title={Learning transferable visual models from natural language supervision},
  author={Radford, Alec and Kim, Jong Wook and Hallacy, Chris and Ramesh, Aditya and Goh, Gabriel and Agarwal, Sandhini and Sastry, Girish and Askell, Amanda and Mishkin, Pamela and Clark, Jack and others},
  booktitle={International conference on machine learning},
  pages={8748--8763},
  year={2021},
  organization={PMLR}
}

@article{loshchilov2017decoupled,
  title={Decoupled weight decay regularization},
  author={Loshchilov, I},
  journal={arXiv preprint arXiv:1711.05101},
  year={2017}
}

@inproceedings{shah-etal-2024-improving,
    title = "Improving {LLM}-based {KGQA} for multi-hop Question Answering with implicit reasoning in few-shot examples",
    author = "Shah, Mili  and
      Cahoon, Joyce  and
      Milletari, Mirco  and
      Tian, Jing  and
      Psallidas, Fotis  and
      Mueller, Andreas  and
      Litombe, Nick",
    editor = "Biswas, Russa  and
      Kaffee, Lucie-Aim{\'e}e  and
      Agarwal, Oshin  and
      Minervini, Pasquale  and
      Singh, Sameer  and
      de Melo, Gerard",
    booktitle = "Proceedings of the 1st Workshop on Knowledge Graphs and Large Language Models (KaLLM 2024)",
    month = aug,
    year = "2024",
    address = "Bangkok, Thailand",
    publisher = "Association for Computational Linguistics",
    url = "https://aclanthology.org/2024.kallm-1.13",
    doi = "10.18653/v1/2024.kallm-1.13",
    pages = "125--135",
}

@inproceedings{mascharka2018transparency,
  title={Transparency by design: Closing the gap between performance and interpretability in visual reasoning},
  author={Mascharka, David and Tran, Philip and Soklaski, Ryan and Majumdar, Arjun},
  booktitle={Proceedings of the IEEE conference on computer vision and pattern recognition},
  pages={4942--4950},
  year={2018}
}

@article{andreas2016learning,
  title={Learning to compose neural networks for question answering},
  author={Andreas, Jacob and Rohrbach, Marcus and Darrell, Trevor and Klein, Dan},
  journal={arXiv preprint arXiv:1601.01705},
  year={2016}
}

@inproceedings{perez2018film,
  title={Film: Visual reasoning with a general conditioning layer},
  author={Perez, Ethan and Strub, Florian and De Vries, Harm and Dumoulin, Vincent and Courville, Aaron},
  booktitle={Proceedings of the AAAI conference on artificial intelligence},
  volume={32},
  number={1},
  year={2018}
}

@article{hudson2019learning,
  title={Learning by abstraction: The neural state machine},
  author={Hudson, Drew and Manning, Christopher D},
  journal={Advances in neural information processing systems},
  volume={32},
  year={2019}
}

@article{hudson2018compositional,
  title={Compositional attention networks for machine reasoning},
  author={Hudson, Drew A and Manning, Christopher D},
  journal={arXiv preprint arXiv:1803.03067},
  year={2018}
}

@inproceedings{endo2023motion,
  title={Motion question answering via modular motion programs},
  author={Endo, Mark and Hsu, Joy and Li, Jiaman and Wu, Jiajun},
  booktitle={International Conference on Machine Learning},
  pages={9312--9328},
  year={2023},
  organization={PMLR}
}

@inproceedings{chen2021meta,
  title={Meta module network for compositional visual reasoning},
  author={Chen, Wenhu and Gan, Zhe and Li, Linjie and Cheng, Yu and Wang, William and Liu, Jingjing},
  booktitle={Proceedings of the IEEE/CVF Winter Conference on Applications of Computer Vision},
  pages={655--664},
  year={2021}
}

@article{wu2024star,
  title={Star: A benchmark for situated reasoning in real-world videos},
  author={Wu, Bo and Yu, Shoubin and Chen, Zhenfang and Tenenbaum, Joshua B and Gan, Chuang},
  journal={arXiv preprint arXiv:2405.09711},
  year={2024}
}

@inproceedings{yu2019activitynet,
  title={Activitynet-qa: A dataset for understanding complex web videos via question answering},
  author={Yu, Zhou and Xu, Dejing and Yu, Jun and Yu, Ting and Zhao, Zhou and Zhuang, Yueting and Tao, Dacheng},
  booktitle={Proceedings of the AAAI Conference on Artificial Intelligence},
  volume={33},
  number={01},
  pages={9127--9134},
  year={2019}
}

@inproceedings{xu2017video,
  title={Video Question Answering via Gradually Refined Attention over Appearance and Motion},
  author={Xu, Dejing and Zhao, Zhou and Xiao, Jun and Wu, Fei and Zhang, Hanwang and He, Xiangnan and Zhuang, Yueting},
  booktitle={ACM Multimedia},
  year={2017}
}

@inproceedings{besta2024graph,
  title={Graph of thoughts: Solving elaborate problems with large language models},
  author={Besta, Maciej and Blach, Nils and Kubicek, Ales and Gerstenberger, Robert and Podstawski, Michal and Gianinazzi, Lukas and Gajda, Joanna and Lehmann, Tomasz and Niewiadomski, Hubert and Nyczyk, Piotr and others},
  booktitle={Proceedings of the AAAI Conference on Artificial Intelligence},
  volume={38},
  number={16},
  pages={17682--17690},
  year={2024}
}

@article{wei2022chain,
  title={Chain-of-thought prompting elicits reasoning in large language models},
  author={Wei, Jason and Wang, Xuezhi and Schuurmans, Dale and Bosma, Maarten and Xia, Fei and Chi, Ed and Le, Quoc V and Zhou, Denny and others},
  journal={Advances in neural information processing systems},
  volume={35},
  pages={24824--24837},
  year={2022}
}

@article{wang2022self,
  title={Self-consistency improves chain of thought reasoning in language models},
  author={Wang, Xuezhi and Wei, Jason and Schuurmans, Dale and Le, Quoc and Chi, Ed and Narang, Sharan and Chowdhery, Aakanksha and Zhou, Denny},
  journal={arXiv preprint arXiv:2203.11171},
  year={2022}
}

@article{yao2023tree,
  title={Tree of thoughts: Deliberate problem solving with large language models},
  author={Yao, Shunyu and Yu, Dian and Zhao, Jeffrey and Shafran, Izhak and Griffiths, Tom and Cao, Yuan and Narasimhan, Karthik},
  journal={Advances in neural information processing systems},
  volume={36},
  pages={11809--11822},
  year={2023}
}

@inproceedings{suris2023vipergpt,
  title={Vipergpt: Visual inference via python execution for reasoning},
  author={Sur{\'\i}s, D{\'\i}dac and Menon, Sachit and Vondrick, Carl},
  booktitle={Proceedings of the IEEE/CVF International Conference on Computer Vision},
  pages={11888--11898},
  year={2023}
}

@article{jaiswal2025learning,
  title={Learning to Reason Iteratively and Parallelly for Complex Visual Reasoning Scenarios},
  author={Jaiswal, Shantanu and Roy, Debaditya and Fernando, Basura and Tan, Cheston},
  journal={Advances in Neural Information Processing Systems},
  volume={37},
  pages={137965--137998},
  year={2025}
}

@article{yu2023self,
  title={Self-chained image-language model for video localization and question answering},
  author={Yu, Shoubin and Cho, Jaemin and Yadav, Prateek and Bansal, Mohit},
  journal={Advances in Neural Information Processing Systems},
  volume={36},
  pages={76749--76771},
  year={2023}
}

@inproceedings{ye2024mplug,
  title={mplug-owl3: Towards long image-sequence understanding in multi-modal large language models},
  author={Ye, Jiabo and Xu, Haiyang and Liu, Haowei and Hu, Anwen and Yan, Ming and Qian, Qi and Zhang, Ji and Huang, Fei and Zhou, Jingren},
  booktitle={The Thirteenth International Conference on Learning Representations},
  year={2024}
}

@article{wu2024deepseek,
  title={Deepseek-vl2: Mixture-of-experts vision-language models for advanced multimodal understanding},
  author={Wu, Zhiyu and Chen, Xiaokang and Pan, Zizheng and Liu, Xingchao and Liu, Wen and Dai, Damai and Gao, Huazuo and Ma, Yiyang and Wu, Chengyue and Wang, Bingxuan and others},
  journal={arXiv preprint arXiv:2412.10302},
  year={2024}
}

@article{bai2025qwen2,
  title={Qwen2. 5-VL Technical Report},
  author={Bai, Shuai and Chen, Keqin and Liu, Xuejing and Wang, Jialin and Ge, Wenbin and Song, Sibo and Dang, Kai and Wang, Peng and Wang, Shijie and Tang, Jun and others},
  journal={arXiv preprint arXiv:2502.13923},
  year={2025}
}

@inproceedings{choudhury2024video,
  title={Video Question Answering with Procedural Programs},
  author={Choudhury, Rohan and Niinuma, Koichiro and Kitani, Kris M and Jeni, L{\'a}szl{\'o} A},
  booktitle={European Conference on Computer Vision},
  pages={315--332},
  year={2024},
  organization={Springer}
}

@inproceedings{ma2022visual,
  title={Visual knowledge graph for human action reasoning in videos},
  author={Ma, Yue and Wang, Yali and Wu, Yue and Lyu, Ziyu and Chen, Siran and Li, Xiu and Qiao, Yu},
  booktitle={Proceedings of the 30th ACM International Conference on Multimedia},
  pages={4132--4141},
  year={2022}
}

@article{ghosh2020all,
  title={All about knowledge graphs for actions},
  author={Ghosh, Pallabi and Saini, Nirat and Davis, Larry S and Shrivastava, Abhinav},
  journal={arXiv preprint arXiv:2008.12432},
  year={2020}
}

@article{lin2019kagnet,
  title={Kagnet: Knowledge-aware graph networks for commonsense reasoning},
  author={Lin, Bill Yuchen and Chen, Xinyue and Chen, Jamin and Ren, Xiang},
  journal={arXiv preprint arXiv:1909.02151},
  year={2019}
}

@inproceedings{sadhu2021visual,
  title={Visual semantic role labeling for video understanding},
  author={Sadhu, Arka and Gupta, Tanmay and Yatskar, Mark and Nevatia, Ram and Kembhavi, Aniruddha},
  booktitle={Proceedings of the IEEE/CVF Conference on Computer Vision and Pattern Recognition},
  pages={5589--5600},
  year={2021}
}

@article{lei2019tvqa+,
  title={Tvqa+: Spatio-temporal grounding for video question answering},
  author={Lei, Jie and Yu, Licheng and Berg, Tamara L and Bansal, Mohit},
  journal={arXiv preprint arXiv:1904.11574},
  year={2019}
}

@article{alayrac2022flamingo,
  title={Flamingo: a visual language model for few-shot learning},
  author={Alayrac, Jean-Baptiste and Donahue, Jeff and Luc, Pauline and Miech, Antoine and Barr, Iain and Hasson, Yana and Lenc, Karel and Mensch, Arthur and Millican, Katherine and Reynolds, Malcolm and others},
  journal={Advances in neural information processing systems},
  volume={35},
  pages={23716--23736},
  year={2022}
}

@article{maaz2023video,
  title={Video-chatgpt: Towards detailed video understanding via large vision and language models},
  author={Maaz, Muhammad and Rasheed, Hanoona and Khan, Salman and Khan, Fahad Shahbaz},
  journal={arXiv preprint arXiv:2306.05424},
  year={2023}
}

@inproceedings{min2024morevqa,
  title={Morevqa: Exploring modular reasoning models for video question answering},
  author={Min, Juhong and Buch, Shyamal and Nagrani, Arsha and Cho, Minsu and Schmid, Cordelia},
  booktitle={Proceedings of the IEEE/CVF Conference on Computer Vision and Pattern Recognition},
  pages={13235--13245},
  year={2024}
}

@inproceedings{hoang2024semi,
  title={Semi-automated Construction of Complex Knowledge Base Question Answering Dataset Using Large Language Model},
  author={Hoang, Lily and Liausvia, Fiona and Liu, Yan and Nguyen, Thanh-Son},
  booktitle={Joint European Conference on Machine Learning and Knowledge Discovery in Databases},
  pages={230--248},
  year={2024},
  organization={Springer}
}

@article{bordes2013translating,
  title={Translating embeddings for modeling multi-relational data},
  author={Bordes, Antoine and Usunier, Nicolas and Garcia-Duran, Alberto and Weston, Jason and Yakhnenko, Oksana},
  journal={Advances in neural information processing systems},
  volume={26},
  year={2013}
}

@inproceedings{wang2014knowledge,
  title={Knowledge graph embedding by translating on hyperplanes},
  author={Wang, Zhen and Zhang, Jianwen and Feng, Jianlin and Chen, Zheng},
  booktitle={Proceedings of the AAAI conference on artificial intelligence},
  volume={28},
  number={1},
  year={2014}
}

@inproceedings{he2016deep,
  title={Deep residual learning for image recognition},
  author={He, Kaiming and Zhang, Xiangyu and Ren, Shaoqing and Sun, Jian},
  booktitle={Proceedings of the IEEE conference on computer vision and pattern recognition},
  pages={770--778},
  year={2016}
}

@inproceedings{pennington2014glove,
  title={Glove: Global vectors for word representation},
  author={Pennington, Jeffrey and Socher, Richard and Manning, Christopher D},
  booktitle={Proceedings of the 2014 conference on empirical methods in natural language processing (EMNLP)},
  pages={1532--1543},
  year={2014}
}

@article{de2020statistical,
  title={From statistical relational to neuro-symbolic artificial intelligence},
  author={De Raedt, Luc and Duman{\v{c}}i{\'c}, Sebastijan and Manhaeve, Robin and Marra, Giuseppe},
  journal={arXiv preprint arXiv:2003.08316},
  year={2020}
}

@article{abbe2022learning,
  title={Learning to reason with neural networks: Generalization, unseen data and boolean measures},
  author={Abbe, Emmanuel and Bengio, Samy and Cornacchia, Elisabetta and Kleinberg, Jon and Lotfi, Aryo and Raghu, Maithra and Zhang, Chiyuan},
  journal={Advances in Neural Information Processing Systems},
  volume={35},
  pages={2709--2722},
  year={2022}
}

@article{socher2013reasoning,
  title={Reasoning with neural tensor networks for knowledge base completion},
  author={Socher, Richard and Chen, Danqi and Manning, Christopher D and Ng, Andrew},
  journal={Advances in neural information processing systems},
  volume={26},
  year={2013}
}

@article{suddendorf1997mental,
  title={Mental time travel and the evolution of the human mind},
  author={Suddendorf, Thomas and Corballis, Michael C and others},
  journal={Genetic Social and General Psychology Monographs},
  volume={123},
  number={2},
  pages={133--168},
  year={1997},
  publisher={Washington, DC: Heldref Publications, c1985-}
}

@inproceedings{lei2018tvqa,
  title={Tvqa: Localized, compositional video question answering},
  author={Lei, Jie and Yu, Licheng and Bansal, Mohit and Berg, Tamara},
  booktitle={Proceedings of the 2018 conference on empirical methods in natural language processing},
  pages={1369--1379},
  year={2018}
}

@inproceedings{li2025imore,
  title={IMoRe: Implicit Program-Guided Reasoning for Human Motion Q\&A},
  author={Li, Chen and Sugandhika, Chinthani and Ee, Yeo Keat and Peh, Eric and Zhang, Hao and Yang, Hong and Rajan, Deepu and Fernando, Basura},
  booktitle={Proceedings of the IEEE/CVF International Conference on Computer Vision},
  pages={12987--12996},
  year={2025}
}

@inproceedings{nguyen2026pkr,
  title={PKR-QA: A Benchmark for Procedural Knowledge Reasoning with Knowledge Module Learning},
  author={Nguyen, Thanh-Son and Yang, Hong and Neoh, Tzeh Yuan and Zhang, Hao and Keat, Ee Yeo and Fernando, Basura},
booktitle={AAAI},
  year={2026}
}

@inproceedings{ee2025deduce,
  title={Deduce and Select Evidences with Language Models for Training-Free Video Goal Inference},
  author={Ee, Yeo Keat and Zhang, Hao and Matyasko, Alexander and Fernando, Basura},
  booktitle={2025 IEEE/CVF Winter Conference on Applications of Computer Vision (WACV)},
  pages={5937--5947},
  year={2025},
  organization={IEEE}
}

@inproceedings{wan2025infer,
  title={Infer human’s intentions before following natural language instructions},
  author={Wan, Yanming and Wu, Yue and Wang, Yiping and Mao, Jiayuan and Jaques, Natasha},
  booktitle={Proceedings of the AAAI Conference on Artificial Intelligence},
  volume={39},
  number={24},
  pages={25309--25317},
  year={2025}
}

@inproceedings{roy2024predicting,
  title={Predicting the next action by modeling the abstract goal},
  author={Roy, Debaditya and Fernando, Basura},
  booktitle={International Conference on Pattern Recognition},
  pages={162--177},
  year={2024},
  organization={Springer}
}
}

\newpage 
\section{Appendix}
\subsection{LLM Prompts}

\begin{itemize}
    \item \Cref{fig:extractingprompt}: Prompt for extracting action and object from step description.
    \item \Cref{fig:purpose_grounding}: Prompt for extracting purposes of tools.
    \item \Cref{fig:programprompt}: Prompt for generating logical programs. 
    \item  \Cref{fig:prompt.rephrasing}: Prompt for rephrasing question.
\end{itemize}

\subsection{Tables and Figures}
\begin{itemize}
    \item \Cref{tab:dataset}: Seventeen traversal templates used to construct the PKR-QA dataset. 
    \item \Cref{fig:illustrate3}: Additional example of Step-by-step reasoning of our KML. 
\end{itemize}

\begin{figure}[h!]
    \centering
\noindent\fbox{%
    \parbox{0.45\textwidth}{%
    You are given a list of steps (each line is a step). For each step, extract the following information if specified: \\ 
1. Action: The verb or action being performed.\\ 
2. Object: The item or thing on which the action is being performed.\\ 
Output as a list of dictionary where each item in the list is the extraction for a step, and each dictionary contains the information extracted for the step. Using this format: [\{``Step": ``the input step", ``action": ``'', ``object": ``''\}].\\ 
Below are example: \\ 
Input steps: \\ 
clean inner wall of container\\ 
place a piece of paper on weighing pan\\ 
place weighing sample on weighing pan and read\\ 
place the board on each side\\ 
Output: [\{``Step": ``clean inner wall of container", ``action": ``clean", ``object": ``inner wall of container"
\}, \\
\{ ``Step": ``place a piece of paper on weighing pan", 
    ``action": ``place",
    ``object": ``piece of paper"\}, \\
\{ ``Step": ``place weighing sample on weighing pan and read",
    ``action": ``place", 
    ``object": ``weighing sample"\}, \\
\{ ``Step": ``place the board on each side",
    ``action": ``place",
    ``object": ``board"\}]\\ 
Input steps:  [\texttt{Steps}] \\ 
Output:
    }%
}
    
    \caption{LLM prompt for extracting the action and object based on the text annotation}
    \label{fig:extractingprompt}
\end{figure}

\begin{figure}[h!]
\centering
\begin{tcolorbox}[colback=gray!5, colframe=gray!80, width=\columnwidth, title=Prompt for \textit{Tool}, fonttitle=\bfseries]
\small
Given a task, step, tool, and a set of options related to the tool's capabilities, your job is to identify and output the relevant options that describe the purpose of using the specified tool in the given step of the task.

Only choose the options that are really relevant. Can output empty list if none relevant options are found. Provide only the relevant options as the output in a list format and rank them from the most relevant to least.\\

Example: \\
Task: ``CookOmelet"\\
Step: ``pour the egg into the bowl"\\
Tool: ``egg"\\
Options: [``make breakfast", ``hiding at easter", ``nurturing", ``make omelet", ``breed", ``cooking", ``rolling on white house lawn", ``food", ``stick to skillet"]\\
Output: [``make breakfast", ``make omelet", ``cooking", ``food"]\\

ACTUAL OUTPUT  \\
Task: \{task\}\\
Step: \{step\}\\
Tool: \{term\}\\
Options: \{options\}\\
Output:     
\end{tcolorbox}
\caption{Prompt template for grounding the \textit{Tool}’s \textit{Purpose}. Similar templates are used for \textit{Object} and \textit{Action}.}
\label{fig:purpose_grounding}
\end{figure}

\begin{figure}[t]
\centering        
\noindent\fbox{%
    \parbox{0.45\textwidth}{
\scriptsize
I have a knowledge graph in the following format.\\
KG Entities: Step, Action, Object, GroundedTool, Purpose, Domain, Task
KG Relations\\
------------\\
$ Step \rightarrow \text{HAS\_TOOL} \rightarrow Tool$ \\
$ Tool \rightarrow HAS\_PURPOSE \rightarrow Purpose $ \\
$ Step \rightarrow HAS\_NEXT\_STEP \rightarrow Step$ \\
$ Domain \rightarrow HAS\_TASK \rightarrow Task$ \\
$ Task \rightarrow HAS\_STEP \rightarrow Step$ \\
$ Step \rightarrow HAS\_ACTION \rightarrow Action$ \\
$ Step \rightarrow HAS\_OBJECT \rightarrow Object$ \\
$ Purpose \rightarrow HAS\_SIMILAR\_PURPOSE \rightarrow Purpose $ \\
$ Tool \rightarrow TOOL\_TO\_STEP \rightarrow Step $ \\
$ Purpose \rightarrow PURPOSE\_TO\_TOOL  \rightarrow Tool $ \\
$ Step \rightarrow HAS\_PREVIOUS\_STEP \rightarrow Step $ \\
$ Task \rightarrow IN\_DOMAIN \rightarrow Domain$ \\ 
$ Step \rightarrow IN\_TASK \rightarrow Task$ \\
$ Action \rightarrow ACTION\_IN\_STEP \rightarrow Step $ \\
$ Object \rightarrow OBJECT\_IN\_STEP \rightarrow Step $ \\
\\
I can recognize the step and the task of the video. From knowing the step and the task, I need to answer the question based on the KG relations. 
I have to traverse from one entity in the KG to the next.
The QUESTION is ``QUESTION".\\

You can answer this step by step. First answer the following question.
What should be the first entity I should recognize in the video? Options are only [Step] or [Task].\\
Then generate the answer, “what is the answer entity for the QUESTION?”\\
Generate the list of relations that allows you to traverse from the [start entity] to the [answer entity].\\
Generate a consistent list of relations. Consistent means if Entity\_A $\rightarrow$ KG Relation\_X $\rightarrow$ Entity\_B is the k-th relation then the k+1-th relation should start from Entity\_B.
At every answer check if you can traverse to the current Entity from the start Entity.
You should be able to traverse the knowledge graph from start entity to answer entity.\\
Now generate the answer in Python dictionary where the first key is the GROUNDED\_ENTITY and generate the correct value.
Then for the second key generate the programs as a python list of traversal nodes. The key name for this is the PROGRAMS. 
The value of PROGRAMS are the list of lists where each list is a valid traversal from the GROUNDED\_ENTITY to the target entity to answer the question. 
You may generate alternative path that allows you to traverse from the start entity to the answer entity in the shortest possible ways also.
Note that each program should contain only the relation name.
Generate Python list of relations for each alternative paths.
Now generate.
    }
}
\newline 
\caption{The program synthesis using LLMs. The used prompt to generate the logical programs of traversal is shown here.}
    \label{fig:programprompt}
\end{figure}

\begin{figure}[th!]
    \centering
    \noindent\fbox{%
    \parbox{0.49\textwidth}{%
    Rephrase the following question: ``[\texttt{original question sentence}]'' in 30 different ways while maintaining its semantic meaning and terms such as tool, step, task, domain.\\ 
    Ensure the rephrasing include a mix of conventional and unconventional sentence structures, such as changing the order of words.\\ 
    List each variation using * as bullet points.
    }%
}
    \caption{The prompt used for rephrasing the questions.}
    \label{fig:prompt.rephrasing}
\end{figure}

\begin{table}[t!]
\scriptsize
\centering
\rowcolors{2}{gray!10}{white}
\begin{tabular}{
|>{\centering\arraybackslash}m{5cm}|             
}
\hline
\raisebox{-.6\height}{\includegraphics[width=5cm]{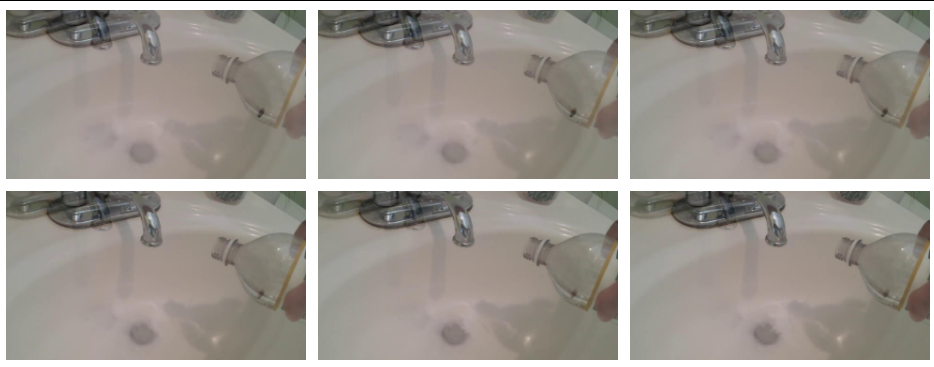}}  
\\
\parbox[t]{\linewidth}{\raggedright \textbf{Question:} What could potentially be the next step? \newline 
\textbf{Task:} Unclog Sink With Baking Soda\newline 
\textbf{Step:} Add hot water to the sink hole.}
\\ \hline
\rowcolor{blue!25}
\textbf{NEXT\_STEP}  \\
\hline
\rowcolor{green!25}
Output=\textbf{Next Step} \\
\hline
add baking soda to the sink hole (0.294) \\

add hot water to the sink hole (0.265) \\

add vinegar to the sink hole (0.255) \\

turn on the water tap to wash (0.246) \\

remove the old faucet (0.169) \\
\hline
\end{tabular}
\caption{
This question demands reasoning over the next step similar to action anticipation.
As can be seen the KML method seem to handle the future uncertainties well and list down the plausible set of next steps.
}
\label{fig:illustrate3}
\end{table}

\begin{table*}[t]
\centering
\begin{tabular}{c|p{7.5cm}|p{5cm}|r}
\toprule
\textbf{No.} & \textbf{Traversal Template} & Sample Question & Test Size \\ \midrule
$\boldsymbol{Q}_{1}$& $\texttt{\textcolor{blue}{Task}} \xrightarrow{\texttt{HAS\_STEP}}\texttt{\textcolor{blue}{Step}} \xrightarrow{\texttt{HAS\_TOOL}} \texttt{\textcolor{red}{Tool}}$                                                                                 &  What \textbf{tool} can be used in this step?             &  5,024    \\\midrule
$\boldsymbol{Q}_{2}$     &  $\texttt{\textcolor{blue}{Task}} \xrightarrow{\texttt{HAS\_STEP}} \texttt{\textcolor{blue}{Step}}  \xrightarrow{\texttt{HAS\_NEXT\_STEP} (\text{freq.})} \texttt{\textcolor{red}{Step}}$                 & What is the \textbf{most} probable \textbf{next} step?                                                          & 5,412  \\\midrule
$\boldsymbol{Q}_{3}$       &   $\texttt{\textcolor{blue}{Task}} \xrightarrow{\texttt{HAS\_STEP}}\texttt{\textcolor{blue}{Step}} \xrightarrow{\texttt{HAS\_NEXT\_STEP}} \texttt{\textcolor{red}{Step}}$               & What could potentially be the \textbf{next} step?                                                         & 5,412  \\\midrule
$\boldsymbol{Q}_{4}$     &     $\texttt{\textcolor{red}{Step}} \xrightarrow{\texttt{HAS\_NEXT\_STEP}} \texttt{\textcolor{blue}{Step}}\xleftarrow{\texttt{HAS\_STEP}}\texttt{\textcolor{blue}{Task}}  $               & What could potentially be the \textbf{preceding} step?                                                   & 5,408   \\\midrule
$\boldsymbol{Q}_{5}$    &     $\texttt{\textcolor{red}{Step}} \xrightarrow{\texttt{HAS\_NEXT\_STEP} (\text{freq.})} \texttt{\textcolor{blue}{Step}} \xleftarrow{\texttt{HAS\_STEP}}\texttt{\textcolor{blue}{Task}} $               & What is the \textbf{most} probable \textbf{preceding} step?           & 5,408  \\\midrule
$\boldsymbol{Q}_{6}$   &   $\texttt{\textcolor{blue}{Task}} \xrightarrow{\texttt{HAS\_STEP}}\texttt{\textcolor{blue}{Step}} \xrightarrow{\texttt{HAS\_NEXT\_STEP}} \texttt{Step}\xrightarrow{\texttt{HAS\_TOOL}} \texttt{\textcolor{red}{Tool}} $                  & What \textbf{tool} could be used in the \textbf{next} step? & 5,236  \\\midrule
$\boldsymbol{Q}_{7}$  &    $\texttt{START} \xrightarrow{\texttt{HAS\_NEXT\_STEP} (\text{freq.})} \texttt{\textcolor{red}{Step}}\xleftarrow{\texttt{HAS\_STEP}}\texttt{\textcolor{blue}{Task}}$              & What is the most probable \textbf{initial step} of this task? & 1,439   \\\midrule
$\boldsymbol{Q}_{8}$  &    $\texttt{\textcolor{blue}{Task}} \xrightarrow{\texttt{HAS\_STEP}} \texttt{\textcolor{red}{Step}}   \xrightarrow{\texttt{HAS\_NEXT\_STEP}\text{(freq.)}} \texttt{END}$             & What is the most probable \textbf{final step} of this task?  & 1,439   \\\midrule
$\boldsymbol{Q}_{9}$  &   $\texttt{\textcolor{blue}{Task}} \xleftarrow{\texttt{HAS\_TASK}} \texttt{\textcolor{red}{Domain}}$              & Which \textbf{domain} does this task belong to? & 1,439 \\\midrule
$\boldsymbol{Q}_{10}$   &   $\texttt{\textcolor{blue}{Task}} \xrightarrow{\texttt{HAS\_STEP}}\texttt{\textcolor{blue}{Step}} \xrightarrow{\texttt{HAS\_TOOL}} \texttt{Tool}\xrightarrow{\texttt{HAS\_PURPOSE}} \texttt{\textcolor{red}{Purpose}}$             & What is the \textbf{purpose} of the \textbf{tool} used in this step?  &2,822   \\\midrule
$\boldsymbol{Q}_{11}$    &   $\texttt{\textcolor{blue}{Task}} \xrightarrow{\texttt{HAS\_STEP}}\texttt{\textcolor{blue}{Step}} \xrightarrow{\texttt{HAS\_ACTION}} \texttt{Action}\xrightarrow{\texttt{HAS\_PURPOSE}} \texttt{\textcolor{red}{Purpose}}$             & What is the \textbf{purpose} of the \textbf{action} in this step?    &258     \\\midrule
$\boldsymbol{Q}_{12}$  &    $\texttt{\textcolor{blue}{Task}} \xrightarrow{\texttt{HAS\_STEP}}\texttt{\textcolor{blue}{Step}} \xrightarrow{\texttt{HAS\_OBJECT}} \texttt{Object}\xrightarrow{\texttt{HAS\_PURPOSE}} \texttt{\textcolor{red}{Purpose}}$             & What is the \textbf{purpose} of the \textbf{object} in this step?     &738     \\\midrule
$\boldsymbol{Q}_{13}$   &    $\texttt{Tool}:T_1 \xrightarrow{\texttt{HAS\_PURPOSE}} \texttt{\textcolor{red}{Purpose}} \textcolor{red}{:P_1} \quad \textbf{where} \quad \texttt{\textcolor{blue}{Task}}\xrightarrow{\texttt{HAS\_STEP}}  \texttt{\textcolor{blue}{Step}} \xrightarrow{\texttt{HAS\_TOOL}} \texttt{Tool}:T1 \xnorightarrow{\texttt{HAS\_PURPOSE}} \texttt{Purpose}:P_1$            & What are the \textbf{additional purposes} for which the tool shown in this video can be used, aside from its intended use in the step demonstrated? & 2,611    \\\midrule
$\boldsymbol{Q}_{14}$  &     $\texttt{Object}:O_1 \xrightarrow{\texttt{HAS\_PURPOSE}} \texttt{\textcolor{red}{Purpose}}\textcolor{red}{:P_1} \quad \textbf{where} \quad \texttt{\textcolor{blue}{Task}}\xrightarrow{\texttt{HAS\_STEP}}  \texttt{\textcolor{blue}{Step}} \xrightarrow{\texttt{HAS\_OBJECT}} \texttt{Object}:O1 \xnorightarrow{\texttt{HAS\_PURPOSE}} \texttt{Purpose}:P_1$             & What are the \textbf{additional purposes} for which the \textbf{object} shown in this video can be used, if I do not want to use it for intended use in the step demonstrated? &674     \\\midrule
$\boldsymbol{Q}_{15}$  &  $\textcolor{red}{\texttt{Tool}:T_1} \xrightarrow{\texttt{HAS\_PURPOSE}} \texttt{Purpose}:P_1 \quad \textbf{where} \quad \texttt{\textcolor{blue}{Task}} \xrightarrow{\texttt{HAS\_STEP}}\texttt{\textcolor{blue}{Step}} \xrightarrow{\texttt{HAS\_TOOL}} \texttt{Tool}:T_2\xrightarrow{\texttt{HAS\_PURPOSE}} \texttt{Purpose}:P_2\xrightarrow{\texttt{HAS\_SIMILAR\_PURPOSE}} \texttt{Purpose}:P_3 \quad \textbf{and} \quad T_1 \neq T_2 \quad \textbf{and} \quad (P_1 = P_2 \quad \textbf{or} \quad P_1 = P_3)$                 & what is an \textbf{alternative tool} that can be used for this step if I don't have the current tool?  &875   \\\midrule
$\boldsymbol{Q}_{16}$  &   $\textcolor{blue}{\texttt{Task:}A_1} \xrightarrow{\texttt{HAS\_STEP}}\texttt{\textcolor{blue}{Step}} \xrightarrow{\texttt{HAS\_TOOL}} \texttt{Tool}:T_1\xrightarrow{\texttt{HAS\_PURPOSE}} \texttt{Purpose} \xleftarrow{\texttt{HAS\_PURPOSE}}\texttt{Tool}:T_1\xleftarrow{\texttt{HAS\_TOOL}}\texttt{Step}\xleftarrow{\texttt{HAS\_STEP}}\textcolor{red}{\texttt{Task}:A_2} \quad \textbf{where} \quad A_1 \neq A_2$              & What can be \textbf{other task} that can use the \textbf{tool} shown in this video for the same purpose?    & 2,513   \\\midrule
$\boldsymbol{Q}_{17}$ &      $\textcolor{blue}{\texttt{Task:}A_1} \xrightarrow{\texttt{HAS\_STEP}}\texttt{\textcolor{blue}{Step}} \xrightarrow{\texttt{HAS\_OBJECT}} \texttt{Object}:O_1\xrightarrow{\texttt{HAS\_PURPOSE}} \texttt{Purpose} \xleftarrow{\texttt{HAS\_PURPOSE}}\texttt{Object}:O_1\xleftarrow{\texttt{HAS\_OBJECT}}\texttt{Step}\xleftarrow{\texttt{HAS\_STEP}}\textcolor{red}{\texttt{Task}:A_2} \quad \textbf{where} \quad A_1 \neq A_2$            & What can be the \textbf{other task} that use the \textbf{object} in this video for the same purpose?     & 213    \\ \midrule
\multicolumn{3}{r|}{\textbf{Total}}                                                                                                                                                                                                                  & 46,921  \\ \bottomrule
\end{tabular}

\caption{Seventeen traversal templates used to construct question-answer instances in the \dataset dataset. For each template, we provide a sample question and the number of test set instances generated. In the traversal templates, \textcolor{blue}{blue text} highlights information grounded in the input video, while \textcolor{red}{red text} marks the target answer node.}
\label{tab:dataset}
\end{table*}

\begin{figure*}[th]
    \centering
    \includegraphics[width=\textwidth]{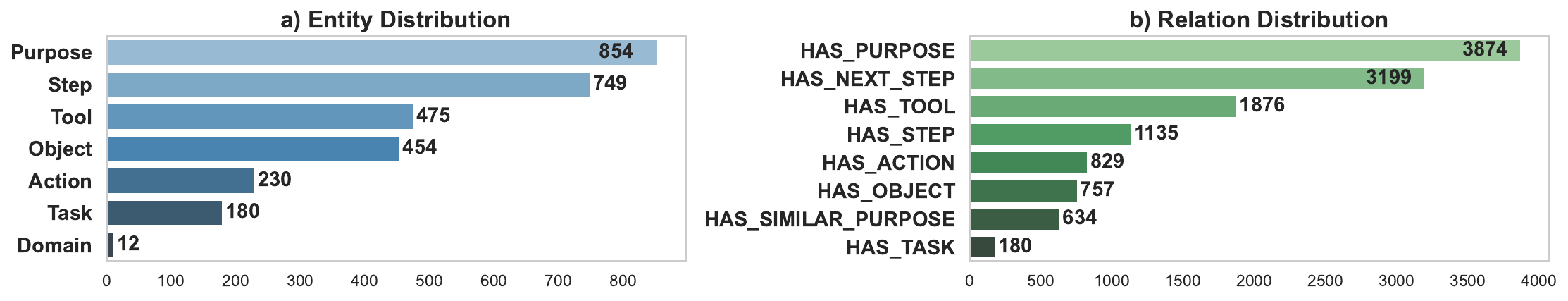}
    \caption{Entity and relation distributions in \pkg. \pkg contains a total of 2,954 unique entities and 12,484 relations.}
    \label{fig:kg_statistic}
\end{figure*}\textbf{}

\end{document}